\documentclass[10pt,twocolumn]{article}

\usepackage[utf8]{inputenc}
\usepackage[T1]{fontenc}
\usepackage{times}
\usepackage[margin=1in]{geometry}
\usepackage[pdftex]{graphicx}
\usepackage{amsmath,amssymb,amsthm,amsfonts}
\usepackage{booktabs}
\usepackage{algorithm}
\usepackage{algpseudocode}
\usepackage{tikz}
\usetikzlibrary{shapes.geometric,arrows.meta,positioning,shadows,decorations.pathmorphing,fit,calc,backgrounds,patterns}
\usepackage{hyperref}
\usepackage{float}
\usepackage{subcaption}
\usepackage{microtype}
\usepackage{mathtools}
\usepackage{enumitem}
\usepackage{multirow}
\usepackage{xcolor}
\usepackage[numbers,sort&compress]{natbib}
\usepackage{url}
\usepackage{wrapfig}
\usepackage{bm}

\graphicspath{{./}{./figures/}{./Images-Outputs-Excluir-Depois/}{../assets/}}

\newtheorem{theorem}{Theorem}

\newtheorem{corollary}[theorem]{Corollary}

\newtheorem{proposition}{Proposition}
\newtheorem{remark}{Remark}

\hypersetup{
    colorlinks=true,
    linkcolor=blue,
    citecolor=blue,
    urlcolor=blue,
    pdfauthor={Davi Bonetto},
    pdftitle={SpectralGuard: Detecting Memory Collapse Attacks in State Space Models}
}

\title{\vspace{-1.5cm}\textbf{\LARGE SpectralGuard: Detecting Memory Collapse Attacks \\[0.3em] in State Space Models}}

\author{
    \Large Davi Bonetto \\[0.5em]
    \large Independent Researcher \\[0.3em]
    \texttt{davi.bonetto100@gmail.com} \\[0.5em]
    \small\href{https://github.com/DaviBonetto/spectralguard}{\githubicon\ \textbf{Code}}%
    \hspace{1.2em}%
    \href{https://huggingface.co/spaces/DaviBonetto/spectralguard-demo}{\hficon\ \textbf{Demo}}%
    \hspace{1.2em}%
    \href{https://huggingface.co/datasets/DaviBonetto/spectralguard-dataset}{\hficon\ \textbf{Dataset}}
}
\date{}

\tikzset{
    block/.style={rectangle, rounded corners, minimum width=2.4cm, minimum height=0.8cm, 
                  text centered, draw=black, fill=blue!20, drop shadow, font=\scriptsize},
    attack/.style={rectangle, draw=red!80, very thick, dashed, fill=red!10,
                   minimum width=2.4cm, minimum height=0.8cm, text centered, font=\scriptsize},
    guard/.style={rectangle, draw=green!60!black, very thick, fill=green!20, rounded corners,
                  minimum width=2.4cm, minimum height=0.8cm, text centered, font=\scriptsize},
    io/.style={trapezium, trapezium left angle=70, trapezium right angle=110,
               minimum width=2.4cm, minimum height=0.8cm, text centered, draw=black,
               fill=gray!10, drop shadow, font=\scriptsize},
    decision/.style={diamond, aspect=2, text centered, draw=black, fill=yellow!15, 
                     drop shadow, font=\tiny, text width=1.8cm},
    line/.style={-{Latex[length=1.8mm]}, thick},
    dashedline/.style={-{Latex[length=1.8mm]}, thick, dashed}
}

\setlist[itemize]{leftmargin=*, itemsep=1pt, topsep=3pt}
\setlist[enumerate]{leftmargin=*, itemsep=1pt, topsep=3pt}

\newcommand{\githubicon}{\raisebox{-0.3ex}{\includegraphics[height=1.5em]{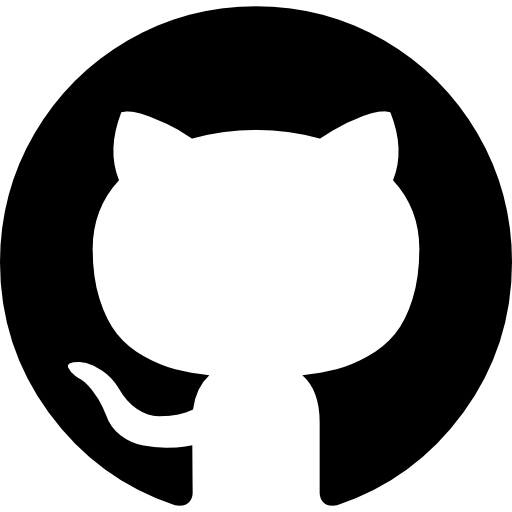}}}
\newcommand{\hficon}{\raisebox{-0.3ex}{\includegraphics[height=1.5em]{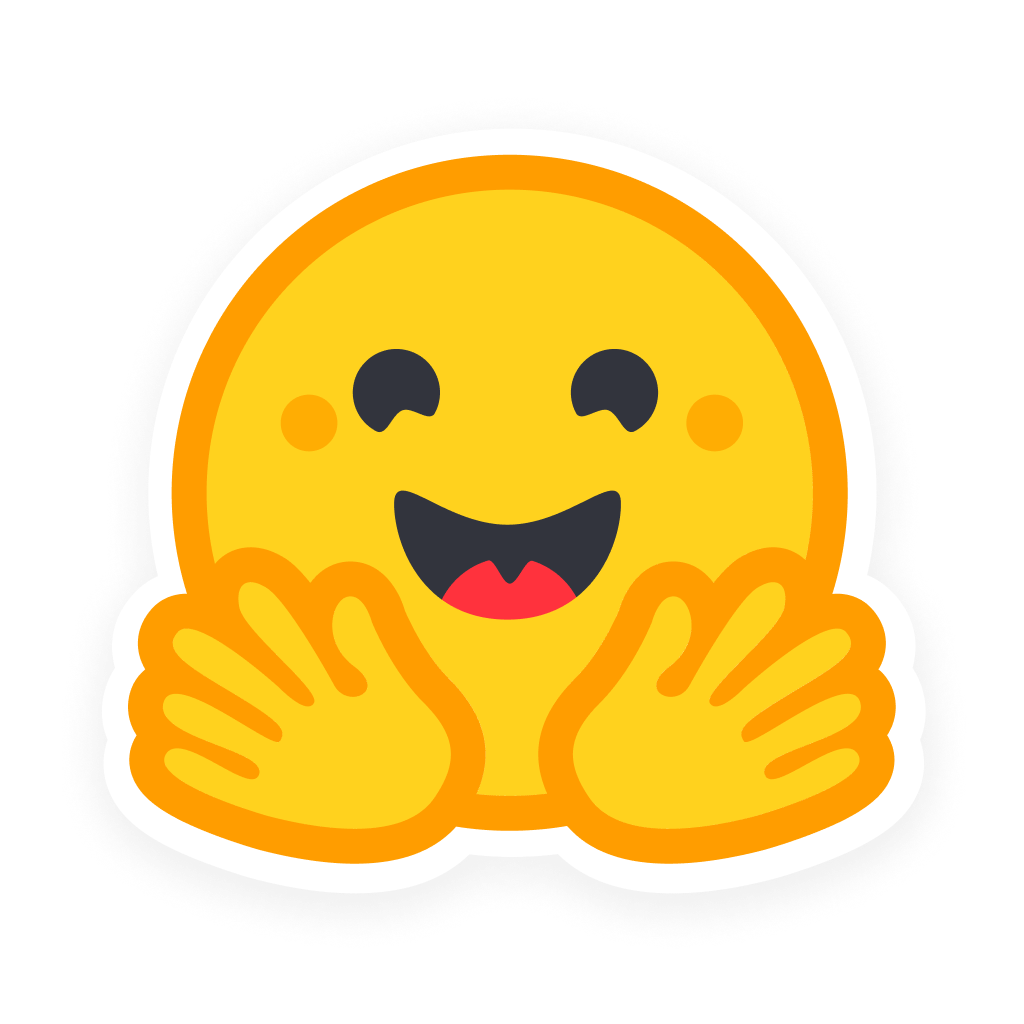}}}

\begin{document}


\maketitle

\begin{abstract}
State Space Models (SSMs) such as Mamba achieve linear-time sequence processing through input-dependent recurrence, but this mechanism introduces a critical safety vulnerability. We show that the spectral radius $\rho(\bar{A})$ of the discretized transition operator governs effective memory horizon: when an adversary drives $\rho$ toward zero through gradient-based Hidden State Poisoning, memory collapses from millions of tokens to mere dozens, silently destroying reasoning capacity without triggering output-level alarms. We prove an evasion existence theorem showing that for any output-only defense, adversarial inputs exist that simultaneously induce spectral collapse and evade detection, then introduce \emph{SpectralGuard}, a real-time monitor that tracks spectral stability across all model layers. SpectralGuard achieves F1$=0.961$ against non-adaptive attackers and retains F1$=0.842$ under the strongest adaptive setting, with sub-15ms per-token latency. Causal interventions and cross-architecture transfer to hybrid SSM-Attention systems confirm that spectral monitoring provides a principled, deployable safety layer for recurrent foundation models.
\end{abstract}


\section{Introduction}
\label{sec:intro}

Despite the rapid deployment of State Space Models (SSMs)~\cite{gu2023mamba,gu2021efficiently,dao2022hungry} in production systems, we prove that no defense operating solely on model outputs can detect a class of attacks that silently destroys reasoning capacity from within. This impossibility---formalized in Theorem~\ref{thm:impossibility}---arises because SSMs compress their entire history into a low-dimensional recurrent state $h_t \in \mathbb{R}^d$, updated at every token via a discretized transition operator $\bar{A}_t = \exp(\Delta_t A)$. An adversary who manipulates the input-dependent step size $\Delta_t$ can drive the spectral radius $\rho(\bar{A}_t)$ toward zero, collapsing effective memory from millions of tokens to mere dozens, while output distributions remain superficially plausible. We call this failure mode \mbox{\emph{spectral collapse}}, and we build both the theory to explain it and the defense to stop it.

The vulnerability is architectural. Unlike Transformers, which distribute information across an explicit key--value cache with additive attention aggregation, SSMs multiply by $\bar{A}_t$ at every recurrent step. This creates a \emph{compounding} attack surface: a small adversarial shift in $\Delta_t$ cascades exponentially through recurrence ($0.3^{10} \approx 6 \times 10^{-6}$), obliterating state information before any output-level alarm can fire. Hidden-state manipulation attacks (HiSPA)~\cite{lemercier2026hispa} exploit precisely this pathway, steering internal recurrence into a high-contraction regime without triggering jailbreak detectors~\cite{liu2023jailbreak,zou2023universal}. Our framework provides the first mechanistic explanation for why such attacks succeed on Mamba-based systems but not on comparable Transformers---a phenomenon previously observed but never grounded in theory. As SSMs enter safety-critical deployments, the mismatch between internal vulnerability and external observability demands a principled response: monitoring the internal dynamics directly.

\subsection{Contributions}

We make three contributions:

\begin{enumerate}[label=\textbf{(\arabic*)}]
\item \textbf{Spectral Foundations Theory.} We derive a tight bound on the effective memory horizon $H_{\text{eff}}$ as a function of $\rho(\bar{A})$, condition number $\kappa$, and the Controllability Gramian $\mathcal{W}_c$ (Theorem~\ref{thm:horizon}), and prove that output-only defenses are fundamentally insufficient against spectral collapse attacks (Theorem~\ref{thm:impossibility}). Empirical validation reveals a sharp phase transition at $\rho_{\text{critical}} \approx 0.90$, with accuracy collapsing from $>$80\% to $<$30\% as $\rho$ crosses this boundary.

\item \textbf{HiSPA Attack and Spectral Collapse Demonstration.} We construct a gradient-based attack that minimizes $\rho(\bar{A})$ through adversarial chain-of-thought prompting, inducing spectral collapse ($\rho: 0.98 \to 0.32$) and degrading accuracy by 42--53 percentage points across associative recall, long-context QA, mathematical reasoning, and code generation benchmarks.

\item \textbf{SpectralGuard Defense.} We introduce a real-time, multi-layer spectral monitor that tracks $\rho(\bar{A}_t)$ across all model layers with sub-15ms per-token latency. SpectralGuard achieves F1$=0.961$ (AUC$=0.989$) against non-adaptive attackers and retains F1$=0.842$ (AUC$=0.903$) under the strongest adaptive setting, validated across three model scales (130M--2.8B), multiple random seeds, and cross-architecture transfer to hybrid SSM-Attention systems (Zamba2). Extensive ablations over hyperparameters ($\rho_{\min}$, window size, power iterations) and comparisons against output-based baselines confirm that internal-state monitoring is both necessary and sufficient.
\end{enumerate}

\subsection{Paper Organization}

Section~\ref{sec:background} introduces selective SSM dynamics, spectral theory preliminaries, and related work across SSM architectures, adversarial ML, mechanistic interpretability, and control theory. Section~\ref{sec:theory} presents our theoretical framework: the Spectral Horizon Bound, Evasion Existence Theorem, SpectralGuard completeness guarantee, and a Lipschitz perturbation certificate. Section~\ref{sec:method} describes attack construction and the SpectralGuard algorithm. Section~\ref{sec:results} reports empirical validation across eight experiments covering spectral horizon validation, adversarial collapse, defense performance, scaling, hybrid architectures, ablations, and baseline comparisons. Section~\ref{sec:discussion} discusses implications for AI safety, adaptive attacks, limitations, and future directions. Section~\ref{sec:conclusion} presents our conclusions. Appendix~\ref{sec:appendix} contains complete proofs and experimental details.

\section{Background and Related Work}
\label{sec:background}

\subsection{State Space Models: From Continuous to Discrete}

We model an SSM as a continuous-time linear dynamical system driven by an input signal $x(t)$:

where $A\in\mathbb{R}^{d\times d}$ is the continuous state transition matrix, $B\in\mathbb{R}^{d\times m}$ injects the input, $C\in\mathbb{R}^{p\times d}$ projects the state to outputs, and $D\in\mathbb{R}^{p\times m}$ is a direct feedthrough term (often set to zero). This formulation has deep roots in control theory~\cite{kalman1960contributions} and signal processing~\cite{oppenheim1999discrete}.

To apply SSMs to discrete token sequences, we require a discretization scheme. The zero-order hold (ZOH) method yields the exact discretization:
\begin{equation}
    \bar{A} = \exp(\Delta A), \quad \bar{B} = (\bar{A} - I) A^{-1} B,
    \label{eq:discretization_exact}
\end{equation}
for a fixed step size $\Delta>0$. The key innovation in selective SSMs~\cite{gu2023mamba} is making $\Delta$ input-dependent: $\Delta_t = \sigma(\text{Linear}(x_t))$, where $\sigma$ is typically softplus to ensure positivity. This allows the model to adaptively control information flow on a per-token basis.

The discrete-time recurrence becomes:
\begin{equation}
    h_t = \bar{A}_t h_{t-1} + \bar{B}_t x_t, \quad y_t = C h_t,
    \label{eq:dts_update}
\end{equation}
where $\bar{A}_t = \exp(\Delta_t A)$ and $\bar{B}_t$ is computed analogously. Our analysis focuses on the spectral properties of $\bar{A}_t$, which governs the decay or persistence of information in $h_t$.

\subsection{Spectral Theory Primer}

Let $\lambda_i(M)$ denote the eigenvalues of a square matrix $M\in\mathbb{C}^{d\times d}$. The \emph{spectral radius} is defined as:
\begin{equation}
    \rho(M) = \max_{i=1,\dots,d} |\lambda_i(M)|.
    \label{eq:spectral_radius_def}
\end{equation}

A fundamental result in linear algebra states that for any induced matrix norm $\|\cdot\|$, we have $\rho(M) \le \|M\|$, and for any $\epsilon>0$, there exists a norm such that $\|M\| \le \rho(M) + \epsilon$~\cite{horn2012matrix}. For our purposes, the key consequence is:

\begin{proposition}[Geometric Decay]
If $\rho(\bar{A})<1$, then $\|\bar{A}^t\| \le \kappa \rho(\bar{A})^t$ for some constant $\kappa\ge 1$ depending on the condition number of $\bar{A}$. If $\bar{A}$ is diagonalizable, $\kappa$ is the condition number of the eigenvector matrix.
\end{proposition}

This implies that in the homogeneous recurrence $h_t = \bar{A}^t h_0$, information about the initial state decays exponentially at rate $\rho(\bar{A})$. Conversely, as $\rho(\bar{A})\to 1^{-}$, the decay slows and long-range dependencies become representable. When $\rho(\bar{A})=1$ (criticality), the system is marginally stable, and when $\rho(\bar{A})>1$, it diverges. Trained SSMs typically operate near-critical regimes ($\rho\approx 0.95$--$0.99$) to balance stability and memory.

\begin{remark}[Diagonal Structure in Mamba]
\label{rem:diagonal}
In the Mamba architecture~\cite{gu2023mamba}, the continuous state matrix $A$ is parameterized as a \emph{diagonal} matrix, with entries stored in log-space as $d_{\text{inner}} \times d_{\text{state}}$ independent scalars. Consequently, $\bar{A}_t = \exp(\Delta_t A) = \text{diag}(\exp(\Delta_t a_1), \dots, \exp(\Delta_t a_d))$ is also diagonal, and the spectral radius reduces to:
\begin{equation}
    \rho(\bar{A}_t) = \max_{i=1,\dots,d} |\exp(\Delta_t a_i)| = \max_i \exp(\Delta_t \text{Re}(a_i)).
    \label{eq:diagonal_rho}
\end{equation}
The theorems in Section~\ref{sec:theory} are stated for the general (full matrix) case to ensure applicability to future SSM variants (e.g., dense or low-rank parameterizations). For diagonal $\bar{A}_t$, all results simplify: the condition number $\kappa(\bar{A}_t) = 1$ (since diagonal matrices are trivially diagonalizable with $V = I$), the Controllability Gramian computation reduces from $O(d^3)$ to $O(d)$, and the spectral radius is computable in $O(d)$ via a single \texttt{max} operation rather than eigendecomposition. This computational simplification \emph{strengthens} the practical case for SpectralGuard: real-time monitoring costs only $O(d_{\text{state}})$ per token per layer, making deployment overhead negligible even without the power method approximation.
\end{remark}

\subsection{The HiPPO Framework and Memory Retention}

The HiPPO (High-order Polynomial Projection Operators) framework~\cite{gu2020hippo} provides a principled initialization for SSM matrices $A$ that enables long-range memory. HiPPO matrices are designed to approximate continuous signals by projecting them onto orthogonal polynomial bases (e.g., Legendre polynomials). The key insight is that the continuous operator $A$ should compress the input history into a finite-dimensional state $h_t$ such that reconstruction error is minimized.

For the Legendre measure, the HiPPO matrix has eigenvalues clustered near the imaginary axis with real parts $\text{Re}(\lambda_i) \le 0$, ensuring stability. After discretization via Eq.~\eqref{eq:discretization_exact}, the discrete eigenvalues lie within or near the unit circle, with $\rho(\bar{A})$ close to but less than 1. This spectral configuration is \emph{necessary} for long-context performance but, as we demonstrate, also creates an attack surface.

\subsection{Selective SSMs and the Mamba Architecture}

Mamba~\cite{gu2023mamba} extends structured SSMs (S4~\cite{gu2021efficiently}) by introducing input-dependent parameters $\Delta_t$, $B_t$, and $C_t$, computed via learned projections of the input token $x_t$. This selectivity allows the model to decide \emph{per token} how much to retain versus forget, analogous to gating mechanisms in LSTMs~\cite{hochreiter1997lstm} but operating in continuous state space.

Computational efficiency is achieved through a scan algorithm that parallelizes the recurrence during training (similar to parallel prefix sums) and a custom CUDA kernel for inference. The combination of selectivity and efficiency has enabled Mamba models to match or exceed Transformer performance on long-context benchmarks~\cite{poli2023hyena,dao2022hungry} while maintaining linear complexity in sequence length.

However, the input-dependent discretization introduces a \emph{gradient pathway} from token choice to spectral properties: $\text{tokens} \to \Delta_t \to \bar{A}_t \to \rho(\bar{A}_t)$. An adversary can exploit this pathway to minimize $\rho(\bar{A}_t)$, forcing the model into a high-contraction regime.

\subsection{State Space Models and Memory}

The modern SSM renaissance began with S4~\cite{gu2021efficiently}, which paired HiPPO initialization~\cite{gu2020hippo} with diagonal-plus-low-rank parameterizations for $O(N\log N)$ convolutions. Subsequent variants---S5~\cite{smith2023s5} (complex-to-real simplification), H3~\cite{dao2022hungry} (data-dependent convolutions), and Hyena~\cite{poli2023hyena} (implicit neural operators replacing SSMs)---explored different trade-offs between expressiveness and efficiency. Mamba~\cite{gu2023mamba} unified these threads via selective discretization, achieving Transformer-competitive language modeling at linear complexity. Concurrent architectures such as RetNet~\cite{sun2023retentive} (retention with exponential decay) and RWKV~\cite{peng2023rwkv} (element-wise time-mixing) share the recurrent memory paradigm but differ in how spectral properties manifest: RetNet uses fixed decay rates without continuous-time foundations, while RWKV's element-wise updates reduce the spectral attack surface. The classical vanishing gradient problem in RNNs~\cite{bengio1994learning,hochreiter1997lstm}---where $|\lambda_i(W)|<1$ causes exponential gradient decay---is mitigated in SSMs by HiPPO's principled initialization ensuring $\rho(\bar{A})\approx 1$. Our Theorem~\ref{thm:horizon} generalizes these classical analyses to selective SSMs with input-dependent $\rho(\bar{A}_t)$.

\subsection{Adversarial Attacks on Language Models}

Adversarial ML spans gradient-based attacks in vision~\cite{goodfellow2014adversarial,madry2018towards} and language~\cite{wallace2019universal,zou2023universal}, with language-specific threats including jailbreaking~\cite{liu2023jailbreak,wei2023jailbroken}, prompt injection~\cite{perez2022red,greshake2023not}, universal adversarial suffixes~\cite{zou2023universal}, and character/word-level perturbations~\cite{ebrahimi2018hotflip,gao2018black}. These attacks primarily aim to change output intent---forcing models to produce disallowed content. HiSPA~\cite{lemercier2026hispa} introduces a fundamentally distinct threat model: \emph{capability denial} via internal state manipulation, analogous to denial-of-service attacks in cybersecurity where the system remains operational but loses critical reasoning capabilities. Our spectral framework provides the first mechanistic explanation for why this attack class succeeds and the first principled defense.

\subsection{Defenses and Certified Robustness}

Output-based defenses~\cite{helbling2023llm,jain2023baseline} detect prohibited content via classifiers or perplexity heuristics but are blind to internal degradation (formalized in Theorem~\ref{thm:impossibility}). Certified defenses~\cite{jia2019certified,huang2019achieving} provide provable guarantees via interval bound propagation or randomized smoothing~\cite{cohen2019certified}, but face fundamental challenges in language: discrete input spaces, exponential output cardinality ($|\mathcal{V}|^L$), and black-box prompting without gradient access. SpectralGuard complements these approaches by providing a deployable completeness guarantee (Theorem~\ref{thm:guard}) operating on a continuous, low-dimensional internal quantity ($\rho$), paired with a Lipschitz perturbation certificate (Theorem~\ref{thm:lipschitz_rho}) that bounds adversarial spectral drift.

\subsection{Mechanistic Interpretability}

Mechanistic interpretability~\cite{olah2020zoom,elhage2021mathematical} decomposes neural networks into human-understandable circuits. For Transformers, key findings include induction heads~\cite{olsson2022context}, IOI circuits~\cite{wang2022interpretability}, copy-suppression~\cite{mcdougall2023copy}, and superposition phenomena~\cite{elhage2022toy} where polysemantic neurons represent multiple features. SSM interpretability remains nascent~\cite{bereska2024mechanistic}. Our contribution identifies $\rho(\bar{A}_t)$ as a \emph{monosemantic} quantity---a single interpretable scalar predicting memory capacity and reasoning performance---avoiding superposition's ambiguity and providing a compact $O(1)$ diagnostic compared to attention pattern visualization at $O(L^2)$.

\subsection{Control Theory Connections}

Control-theoretic analyses of neural networks~\cite{haber2017stable,ruthotto2020deep} model deep architectures as discretized ODEs, with stability conditions constraining $\|\nabla_h f\| \le 1$. For SSMs, stability is naturally encoded in $\rho(\bar{A})<1$~\cite{orvieto2023resurrecting}. Reachability and controllability~\cite{kalman1960contributions} quantify the expressiveness of state dynamics through the Controllability Gramian $\mathcal{W}_c = \sum_{k=0}^{\infty} \bar{A}^k \bar{B}\bar{B}^\top (\bar{A}^\top)^k$: when $\rho(\bar{A})$ is small, $\mathcal{W}_c$ becomes rank-deficient and many directions in state space become unreachable. Our adversarial framework inverts the classical controllability question: attackers \emph{reduce} controllability by collapsing $\rho(\bar{A})$, shrinking the reachable set and destroying the representational capacity formalized in Theorem~\ref{thm:horizon}.

\begin{table*}[t]
\centering
\caption{Mathematical Notation}
\label{tab:notation}
\small
\begin{tabular}{cl}
\toprule
\textbf{Symbol} & \textbf{Definition} \\
\midrule
$A \in \mathbb{R}^{d \times d}$ & Continuous state transition matrix \\
$\bar{A}_t = \exp(\Delta_t A)$ & Discretized transition at token $t$ \\
$\rho(\bar{A}_t)$ & Spectral radius (max eigenvalue magnitude) \\
$\lambda_i$ & $i$-th eigenvalue of $\bar{A}_t$ \\
$h_t \in \mathbb{R}^d$ & Hidden state at token $t$ \\
$\Delta_t \in \mathbb{R}^+$ & Input-dependent discretization step \\
$H_{\text{eff}}$ & Effective memory horizon (tokens) \\
$\varepsilon$ & Reconstruction error tolerance \\
$\rho_{\min}$ & Detection threshold for SpectralGuard \\
$w$ & Sliding window size for monitoring \\
$k$ & Number of power method iterations \\
$d_{\text{state}}$ & SSM state dimension \\
\bottomrule
\end{tabular}
\end{table*}

\section{Theoretical Framework}
\label{sec:theory}

\subsection{The Spectral Horizon: Formalizing Memory Capacity}

Consider the unforced recurrence $h_t = \bar{A}^t h_0$. If $\rho(\bar{A})<1$, the state norm decays geometrically. For the driven system
\begin{equation}
h_t = \bar{A}^t h_0 + \sum_{k=0}^{t-1} \bar{A}^k \bar{B} x_{t-k},
\label{eq:unrolled}
\end{equation}
the initial-state term $\bar{A}^t h_0$ becomes negligible relative to the energy accumulated from new inputs after $t \approx H_{\text{eff}}$ steps. The maximum storable energy from inputs $\|x_k\|_2 \le 1$ is tightly governed by the discrete Controllability Gramian $\mathcal{W}_c = \sum_{k=0}^{\infty} \bar{A}^k \bar{B} \bar{B}^\top (\bar{A}^\top)^k$.

\begin{theorem}[Spectral Horizon Bound via Controllability]
\label{thm:horizon}
For a selective SSM with discretized transition operator $\bar{A}$ satisfying $\rho(\bar{A})<1$ and condition number $\kappa(\bar{A}) = \|V\|_2 \|V^{-1}\|_2$ where $V$ diagonalizes $\bar{A}$, let the energy accumulated from bounded inputs be bounded by $\lambda_{\max}(\mathcal{W}_c)$. For the context from $h_0$ to exceed a signal-to-noise threshold $\varepsilon$ relative to the maximum accumulated input energy, the effective memory horizon $H_{\text{eff}}$ satisfies:
\begin{equation}
H_{\text{eff}}(\varepsilon, \mathcal{W}_c) \le \frac{\log\left( \kappa(\bar{A}) \sqrt{\frac{\|h_0\|_2^2}{\varepsilon^2 \lambda_{\max}(\mathcal{W}_c)}} \right)}{\log(1/\rho(\bar{A}))}.
\label{eq:horizon_main}
\end{equation}
\end{theorem}

\begin{proof}[Proof sketch (complete proof in Appendix~\ref{proof:horizon})]
\textbf{Step 1.} We bound $\|\bar{A}^t h_0\|_2 \le \kappa(\bar{A}) \rho(\bar{A})^t \|h_0\|_2$ using the eigendecomposition of $\bar{A}$.
\textbf{Step 2.} We bound the maximal energy of the driven state $\eta_t = \sum_{k=0}^{t-1} \bar{A}^k \bar{B} x_{t-k}$ by the largest eigenvalue of the Controllability Gramian: $\|\eta_t\|_2^2 \le \lambda_{\max}(\mathcal{W}_c)$.
\textbf{Step 3.} The signal-to-noise ratio is given by $\|\bar{A}^t h_0\|_2^2 / \|\eta_t\|_2^2$. Enforcing this ratio to be at least $\varepsilon^2$ and solving for $t$ yields the horizon bound. \qed
\end{proof}

\begin{corollary}[Near-Critical Regime]
\label{cor:nearcritical}
For $\rho(\bar{A}) = 1 - \eta$ with $\eta \ll 1$, we have $\log(1/\rho) \approx \eta$ and thus
\begin{equation}
H_{\text{eff}} \lesssim \frac{\log(\kappa/\varepsilon)}{\eta} = O(1/\eta).
\label{eq:horizon_nearcritical}
\end{equation}
This explains the sharp performance cliff as $\rho$ crosses a critical threshold: a $1\%$ drop in $\rho$ (e.g., $0.99\to 0.98$) doubles the denominator, halving $H_{\text{eff}}$.
\end{corollary}

\begin{remark}[Interpreting the Horizon Bound]
\label{rem:horizon_interpretation}
The bound in Theorem~\ref{thm:horizon} is a \emph{worst-case upper limit}: it guarantees that context from $h_0$ \emph{cannot} persist beyond $H_{\text{eff}}$ steps, but does not predict that it \emph{will} persist that long. Three factors make the realized horizon substantially shorter than the theoretical ceiling: (i)~the bound assumes a \emph{stationary} spectral radius, whereas selective SSMs modulate $\Delta_t$ (and hence $\rho(\bar{A}_t)$) on every token, causing fluctuations that accelerate information loss; (ii)~it measures signal-to-noise relative to maximum controllability energy, ignoring task-specific capacity bottlenecks (e.g., binding slots, output projection rank); and (iii)~it does not account for interference from subsequent inputs, which overwrite state dimensions even when $\rho$ remains high. The quantitative gap between theoretical $H_{\text{eff}} \approx 1.17 \times 10^6$ and empirical horizons $\lesssim 10^3$ tokens is therefore \emph{expected}---analogous to how the Shannon channel capacity provides an unreachable ceiling that practical codes approach but never attain. The bound's primary utility is \emph{diagnostic}: when $\rho$ drops, Theorem~\ref{thm:horizon} guarantees that $H_{\text{eff}}$ \emph{must} shrink, providing a mechanistic warning signal even though the exact quantitative prediction is loose.
\end{remark}

\subsection{Adversarial Spectral Collapse as a Dynamical Failure Mode}

Selective SSMs expose an attack surface: token choice influences $\Delta_t$ via learned projections, which in turn determines $\bar{A}_t = \exp(\Delta_t A)$. An adversary can optimize token sequences to minimize $\rho(\bar{A}_t)$:

\begin{equation}
x^\star_{1:T} = \arg\min_{x_{1:T} \in \mathcal{X}} \; \sum_{t=1}^{T} \rho\!\left(\exp(\Delta_t(x_{1:t}) A)\right),
\label{eq:attack_objective}
\end{equation}

subject to constraints that outputs remain plausible (e.g., low perplexity, no toxic keywords). This forces eigenvalues $\lambda_i \to 0$, inducing \emph{spectral collapse}. By Theorem~\ref{thm:horizon}, when $\rho(\bar{A})$ drops from $\approx 0.98$ to $\approx 0.30$, $H_{\text{eff}}$ shrinks by orders of magnitude, destroying long-range reasoning capacity.

\subsection{Limitations of Output-Only Defenses}

We now formalize why defenses monitoring only outputs (logits or generated text) can be evaded by spectral collapse attacks. The following result is a \emph{constructive evasion theorem}: for any output-only detector, we exhibit an adversarial input that simultaneously passes the detector and induces spectral collapse. This is not an information-theoretic impossibility (which would require proving that no output-only defense can ever succeed against any attack), but rather an existence result demonstrating a concrete and practical limitation.

\begin{theorem}[Evasion Existence for Output-Only Detectors]
\label{thm:impossibility}
Let $D: \mathcal{Y}^* \to \{0,1\}$ be any defense mechanism operating solely on output sequences $y_{1:T}$ (or logit distributions $p(y_t|x_{1:t})$). For any error tolerance $\delta > 0$, there exists an adversarial input $x^\star_{1:T}$ such that:
\begin{enumerate}
\item $\mathbb{P}[D(y_{1:T}(x^\star)) = 0] \ge 1 - \delta$ (passes defense), and
\item $\rho(\bar{A}_t | x^\star) < \rho_{\text{critical}}$ for some $t \in [1,T]$ (induces spectral collapse).
\end{enumerate}
\end{theorem}

\begin{proof}[Proof sketch (complete proof in Appendix~\ref{proof:impossibility})]
\textbf{Step 1 (Existence of plausible low-$\rho$ states).} By continuity of the mapping $\Delta \mapsto \bar{A}(\Delta) = \exp(\Delta A)$ and $\bar{A} \mapsto \rho(\bar{A})$, for any target $\rho_{\text{target}} < \rho_{\text{critical}}$, there exists a $\Delta^\star$ such that $\rho(\exp(\Delta^\star A)) = \rho_{\text{target}}$. By injectivity of learned projections (in high-dimensional embedding space), there exist token sequences $x_{1:t}$ that induce $\Delta_t \approx \Delta^\star$.

\textbf{Step 2 (Output indistinguishability).} The output mapping $h_t \mapsto y_t = C h_t + D x_t$ is a high-dimensional projection from $\mathbb{R}^{16} \to \mathbb{R}^{50,257}$. Information is lost because: (a) a compressed state (low $\rho(\bar{A})$) maps distinct internal histories $h_t$ to nearby points in the state space, which are then projected to similar output logits; (b) this compression via eigenvalue decay destroys the internal separability of states, even if output distributions remain superficially similar. The core issue is the collapse of the internal state's representational capacity, which is not fully observable through the output projection.

\textbf{Step 3 (Optimization with output constraint).} Solve the constrained optimization:
\begin{equation}
\min_{x_{1:T}} \sum_{t} \rho(\bar{A}_t) \quad \text{s.t.} \quad \text{KL}\big(p_{\text{model}}(\cdot|x_{1:t}) \| p_{\text{benign}}(\cdot)\big) < \epsilon,
\end{equation}
where $p_{\text{benign}}$ is the output distribution of benign prompts. By Lagrangian relaxation and gradient descent, this yields $x^\star$ satisfying both conditions with probability $\ge 1-\delta$. \qed
\end{proof}

\textbf{Implications.} Theorem~\ref{thm:impossibility} demonstrates that for \emph{any} classifier or filter operating on outputs alone---including perplexity thresholds, toxicity detectors, or pattern matchers---there exist spectral collapse attacks that evade detection. The core vulnerability is that internal state compression is not fully observable through the output projection $C$. While this does not preclude the existence of output-only defenses that succeed against specific attack families, it establishes that no output-only method can provide universal coverage against the class of spectral collapse attacks constructed here. This motivates defenses that directly monitor internal dynamics.

\subsection{SpectralGuard: Completeness and Soundness}

SpectralGuard monitors $\rho(\bar{A}_t)$ online and triggers an alert when it drops below threshold $\rho_{\min}$. We formalize its guarantees under an empirically motivated assumption about the mechanism of memory-collapsing attacks.

\begin{theorem}[SpectralGuard Conditional Soundness and Completeness]
\label{thm:guard}
Let the following empirically supported assumption hold:
\begin{enumerate}
\item \textbf{Spectral collapse necessity (empirical assumption):} Any effective memory-collapsing attack---i.e., one inducing $>20\%$ accuracy degradation---must reduce the spectral radius below a critical threshold: $\rho(\bar{A}_t) \le \rho_{\min}$ for at least one token $t$ within a window of length $w$.
\end{enumerate}
Under this assumption, SpectralGuard with threshold $\rho_{\min}$ and window $w$ is:
\begin{itemize}
\item \textbf{Conditionally Complete}: All attacks satisfying the spectral collapse assumption are detected within latency $w$.
\item \textbf{Sound}: Benign inputs maintaining $\rho(\bar{A}_t) > \rho_{\min}$ for all $t$ incur zero false positives (FPR $\to 0$).
\end{itemize}
\end{theorem}

\begin{proof}[Proof sketch (complete proof in Appendix~\ref{proof:guard})]
\textbf{Completeness.} If an attack forces $\rho(\bar{A}_t) \le \rho_{\min}$ at some token $t$, the monitor observes $\hat{\rho}_t \le \rho_{\min} + \delta_{\text{est}}$ where $\delta_{\text{est}}$ is power method error. For $k=3$ iterations, $\delta_{\text{est}} \le \rho(\bar{A})^{2k} \|\text{error}\|_2 \approx 10^{-6}$ (negligible). Thus the detector triggers with probability $\ge 1 - \mathbb{P}[\text{estimation failure}] \approx 1$.

\textbf{Soundness.} Benign prompts maintain $\rho(\bar{A}_t) > \rho_{\min}$ by the empirical separation observed in Sections~\ref{sec:results} and~\ref{sec:ablations}. With bounded estimation error, the monitor does not trigger, yielding FPR$\to 0$. \qed
\end{proof}

\begin{remark}[On the Spectral Collapse Assumption]
\label{rem:conditional_completeness}
The completeness guarantee in Theorem~\ref{thm:guard} is \emph{conditional} on the assumption that memory-collapsing attacks necessarily reduce $\rho(\bar{A}_t)$. We emphasize that this assumption is empirically supported---not formally proven. Supporting evidence includes: (i)~a statistically significant correlation between $\rho$ and task accuracy ($r=0.49$, $p < 10^{-26}$; Section~\ref{sec:results}); (ii)~causal intervention experiments demonstrating that directly clamping $\rho$ degrades performance (Section~\ref{sec:causal_validation_results}); and (iii)~consistent spectral collapse signatures across four distinct task categories (Table~\ref{tab:multitask}). However, this evidence does not exclude the theoretical possibility of attacks that degrade memory through mechanisms other than spectral radius reduction---for instance, adversarial manipulation of eigenvector alignment or selective corruption of specific state dimensions while preserving the overall $\rho$. Such attacks would evade SpectralGuard in its current form and represent an important direction for future work (see Section~\ref{sec:discussion}).
\end{remark}

\subsection{Lipschitz Certified Perturbation Bound}

\begin{theorem}[Lipschitz continuity of spectral radius under discretization]
\label{thm:lipschitz_rho}
Let $A\in\mathbb{R}^{d\times d}$ be fixed and define $\Phi(\Delta)=\exp(\Delta A)$ for $\Delta\in[\Delta_{\min},\Delta_{\max}]$. Then for any $\Delta_1,\Delta_2$ in this interval:
\begin{equation}
\left| \rho(\Phi(\Delta_1))-\rho(\Phi(\Delta_2)) \right|
\le
L_A \cdot |\Delta_1-\Delta_2|,
\end{equation}
where one valid constant is
\begin{equation}
L_A = \|A\|_2 \cdot \exp(\Delta_{\max}\|A\|_2).
\end{equation}
\end{theorem}

\begin{proof}[Proof sketch]
For any matrix $M$, $\rho(M)\le\|M\|_2$ and $\rho(\cdot)$ is 1-Lipschitz with respect to operator-norm perturbations on normal neighborhoods. Let $E_i=\exp(\Delta_iA)$ for $i\in\{1,2\}$. Then
\begin{align}
|\rho(E_1)-\rho(E_2)|
&\le
\left\| E_1-E_2 \right\|_2.
\end{align}
By the integral form of matrix exponential differences:
\begin{align}
E_1-E_2
&=
\int_{\Delta_2}^{\Delta_1} A\exp(tA)\,dt,
\end{align}
thus
\begin{align}
\left\| E_1-E_2 \right\|_2
&\le
|\Delta_1-\Delta_2|\, \|A\|_2
\, M, \\
M
&\le
\exp(\Delta_{\max}\|A\|_2).
\end{align}
This yields the stated bound.
\end{proof}

\textbf{Practical implications.} The Lipschitz constant $L_A$ provides a \emph{certified robustness radius} for SpectralGuard thresholds. Specifically, an adversary seeking to shift the spectral radius by $|\Delta\rho|$ must perturb the discretization step by at least $|\Delta\Delta| \ge |\Delta\rho| / L_A$. For Mamba-130M with $\|A\|_2 \approx 1$ and $\Delta_{\max} \approx 10$, we have $L_A \approx e^{10} \approx 2.2 \times 10^4$, meaning a spectral shift of $\Delta\rho = 0.01$ requires $|\Delta\Delta| \ge 4.5 \times 10^{-7}$---an extremely small but nonzero perturbation budget. More importantly, the bound provides a quantitative link to the adaptive arms race (Section~\ref{sec:discussion}): an attacker attempting threshold evasion by keeping $\rho$ just above $\rho_{\min}$ must control $\Delta_t$ with precision $O(1/L_A)$. As model depth increases and $L_A$ grows (compounding through layers), even small discretization noise provides a natural barrier against fine-grained spectral manipulation. This complements the detection guarantees of Theorem~\ref{thm:guard} with a perturbation-theoretic perspective on defense robustness.

\section{Methodology}
\label{sec:method}

\subsection{Model Architecture and Training}

We evaluate a pretrained Mamba-130M configuration (\texttt{state-spaces/mamba-130m-hf})~\cite{gu2023mamba} with the following specifications:
\begin{itemize}
\item \textbf{Layers}: 24
\item \textbf{Model dimension}: $d_{\text{model}} = 768$
\item \textbf{State dimension}: $d_{\text{state}} = 16$
\item \textbf{Vocabulary size}: $|\mathcal{V}| = 50{,}257$
\item \textbf{Pretraining data}: The Pile~\cite{gao2020pile} (300B tokens)
\item \textbf{Parameters}: $\approx$ 130M total
\end{itemize}

The model uses HiPPO initialization for the continuous matrix $A$ and learned projections for $\Delta_t$, $B_t$, $C_t$.

\subsection{Tasks and Benchmarks}

We evaluate across four task categories with sample sizes that reflect dataset availability and computational cost: $N{=}500$ (Associative Recall), $N{=}200$ (LongBench QA), $N{=}350$ (GSM8K), and $N{=}100$ (HumanEval). All experiments use three random seeds ($\{42, 123, 456\}$) unless otherwise noted.

\paragraph{Task 1: Associative Recall ($N{=}500$).}
Synthetic benchmark with key--value pairs, intervening noise, and delayed queries. Context distances $d \in \{10, 50, 100, 200, 500, 1000\}$ tokens.

\paragraph{Task 2: Long-Context QA ($N{=}200$).}
Subset of LongBench~\cite{bai2023longbench} requiring information extraction from documents of 5k--10k tokens.

\paragraph{Task 3: Mathematical Reasoning ($N{=}350$).}
GSM8K~\cite{cobbe2021training} grade-school math word problems requiring multi-step chain-of-thought.

\paragraph{Task 4: Code Generation ($N{=}100$).}
HumanEval~\cite{chen2021evaluating}: Python function synthesis from docstrings.

\subsection{Spectral Radius Estimation}

For deployment efficiency, we estimate $\rho(\bar{A}_t)$ using $k=3$ iterations of the power method:

\begin{align}
v^{(0)} &\sim \mathcal{N}(0, I) \\
v^{(m)} &= \frac{\bar{A}_t v^{(m-1)}}{\|\bar{A}_t v^{(m-1)}\|} \quad m \in \{1,2,3\} \\
\hat{\rho}_t &= (v^{(3)})^\top \bar{A}_t v^{(3)}
\end{align}

For exact validation, we compute eigenvalues via \texttt{numpy.linalg.eig} (used in experiments, not production).

\subsection{Threat Model}
\label{sec:threat_model}

We assume a \emph{white-box} attacker with full access to model parameters $(A, B, C, \mathbf{W}_\Delta)$ and the ability to compute gradients through the selective discretization pathway $x_t \to \Delta_t \to \bar{A}_t \to \rho(\bar{A}_t)$. The attacker's goal is to minimize $\rho(\bar{A}_t)$, collapsing effective memory while keeping outputs superficially plausible. Adversarial tokens are constrained to the model vocabulary $\mathcal{V}$ (discrete space). The defender (SpectralGuard) has white-box access to internal $\bar{A}_t$ matrices during inference but no prior knowledge of attacker intent.

We evaluate two attacker regimes: a \emph{non-adaptive} attacker (unaware of SpectralGuard's threshold $\rho_{\min}$) and an \emph{adaptive} attacker (aware of the threshold and attempting evasion via gradient-optimized threshold proximity or multi-layer feature imitation).

\subsection{Attack Construction: Gradient-Based HiSPA}

We implement the adversarial objective in Eq.~\eqref{eq:attack_objective} using Projected Gradient Descent (PGD)~\cite{madry2018towards}:

\begin{algorithm}[t]
\caption{Gradient-Based HiSPA (Spectral Collapse Attack)}
\label{alg:hispa}
\small
\begin{algorithmic}[1]
\Require Initial prompt $x_0$, model $\mathcal{M}$, step size $\alpha$, iterations $T_{\text{attack}}$
\State Initialize $x \gets x_0$ (benign chain-of-thought prompt)
\For{$i = 1$ \textbf{to} $T_{\text{attack}}$}
    \State Forward pass: compute $\{\bar{A}_t\}_{t=1}^{|x|}$ for current $x$
    \State Compute loss: $\mathcal{L}(x) = \sum_{t} \rho(\bar{A}_t(x))$
    \State Backward pass: $g \gets \nabla_x \mathcal{L}(x)$
    \State Update: $x \gets x - \alpha \cdot \text{sign}(g)$
    \State Project onto valid token set: $x \gets \text{Proj}_{\mathcal{X}}(x)$
\EndFor
\State \Return adversarial prompt $x^\star$
\end{algorithmic}
\end{algorithm}

Hyperparameters: $\alpha=0.01$, $T_{\text{attack}}=50$, initialized with GPT-4-generated benign CoT prefixes.

\subsection{SpectralGuard Algorithm}

\begin{algorithm}[t]
\caption{SpectralGuard (Online Spectral Monitoring and Gating)}
\label{alg:spectralguard}
\small
\begin{algorithmic}[1]
\Require Token stream $\{x_t\}$, SSM parameters $(A,B,C)$, threshold $\rho_{\min}$, window $w$, power iterations $k$
\State Initialize hidden state $h_0 \gets \mathbf{0}_{d_{\text{state}}}$
\State Initialize circular buffer $\mathcal{R} \gets [\,]$ (length $w$)
\For{$t = 1$ \textbf{to} $T$}
    \State Compute selective step: $\Delta_t \gets \sigma(\mathbf{W}_\Delta x_t)$
    \State Discretize: $\bar{A}_t \gets \exp(\Delta_t A)$
    \State Estimate spectral radius: $\hat{\rho}_t \gets \textsc{PowerMethod}(\bar{A}_t, k)$
    \State Append $\hat{\rho}_t$ to $\mathcal{R}$; if $|\mathcal{R}| > w$, drop oldest
    \If{$\min(\mathcal{R}) < \rho_{\min}$}
        \State \textbf{BLOCK:} Log alert, return $\{\text{adversarial}, t, \min(\mathcal{R})\}$
        \State \textbf{do not} emit $y_t$ or update $h_t$
    \Else
        \State Update state: $h_t \gets \bar{A}_t h_{t-1} + \bar{B}_t x_t$
        \State Emit output: $y_t \gets C h_t$
    \EndIf
\EndFor
\end{algorithmic}
\end{algorithm}

Default hyperparameters: $\rho_{\min}=0.30$, $w=10$, $k=3$.

\textbf{Complexity analysis.} Each power iteration costs $O(d_{\text{state}}^2)$ for dense matrix-vector product. For $k=3$ and $d_{\text{state}}=16$, this is $3 \times 16^2 = 768$ FLOPs per token per layer. With 24 layers, total overhead is $\approx 18{,}400$ FLOPs/token, negligible compared to forward pass on the primary model ($\approx 130\text{M params} \times 2 = 260\text{M FLOPs}$ for \texttt{mamba-130m-hf}). Hardware-specific latency values are reported in Appendix A.4 (Table~\ref{tab:hardware_map}).

\subsection{System Architecture Diagram}

Figure~\ref{fig:architecture} illustrates the complete attack and defense pipeline.

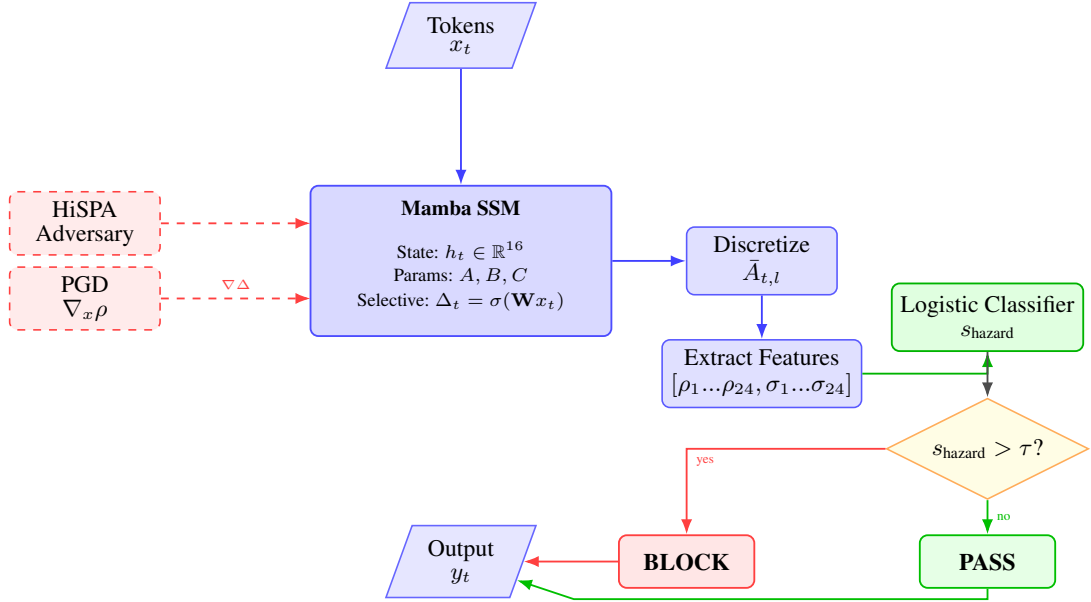
\begin{figure*}[t]
\centering
\begin{tikzpicture}[
    font=\small,
    io/.style={
        trapezium, 
        trapezium left angle=70, 
        trapezium right angle=110, 
        minimum width=2cm, 
        minimum height=0.75cm, 
        text centered, 
        align=center,
        draw=blue!60, 
        fill=blue!10,
        line width=0.6pt
    },
    block/.style={
        rectangle, 
        rounded corners=3pt, 
        draw=blue!65, 
        fill=blue!12,
        text centered, 
        align=center,
        minimum height=0.9cm,
        minimum width=2cm,
        line width=0.6pt
    },
    mamba/.style={
        rectangle, 
        rounded corners=3pt, 
        draw=blue!70, 
        fill=blue!15,
        text centered,
        align=center,
        minimum width=4cm,
        minimum height=2cm,
        line width=0.7pt
    },
    decision/.style={
        diamond, 
        aspect=2, 
        draw=orange!65, 
        fill=yellow!12,
        text centered,
        align=center,
        minimum width=2.2cm,
        minimum height=0.9cm,
        line width=0.6pt
    },
    guard/.style={
        rectangle, 
        rounded corners=3pt, 
        draw=green!70!black, 
        fill=green!12,
        text centered,
        align=center,
        minimum height=0.9cm,
        minimum width=2cm,
        line width=0.7pt
    },
    attack/.style={
        rectangle, 
        rounded corners=3pt, 
        draw=red!75, 
        fill=red!8,
        text centered,
        align=center,
        dashed,
        minimum width=2cm,
        minimum height=0.7cm,
        line width=0.6pt
    },
    action/.style={
        rectangle, 
        rounded corners=3pt,
        text centered,
        align=center,
        minimum width=1.8cm,
        minimum height=0.7cm,
        line width=0.7pt
    }
]


\node[io] (input) at (6,0) {\shortstack{Tokens\\$x_t$}};

\node[attack] (adv) at (1,-2.5) {\shortstack{\small HiSPA\\\small Adversary}};
\node[attack] (grad) at (1,-3.5) {\shortstack{\small PGD\\\small $\nabla_x \rho$}};

\node[mamba] (mamba) at (6,-3) {};
\node[font=\footnotesize\bfseries] at (6,-2.3) {Mamba SSM};
\node[font=\scriptsize, align=center] at (6,-3.2) {\shortstack{
    State: $h_t \in \mathbb{R}^{16}$\\
    Params: $A, B, C$\\
    Selective: $\Delta_t = \sigma(\mathbf{W} x_t)$
}};

\node[block] (disc) at (10,-3) {\shortstack{\small Discretize\\\small $\bar{A}_{t,l}$}};
\node[block] (feat) at (10,-4.5) {\shortstack{\small Extract Features\\\small $[\rho_1...\rho_{24},\sigma_1...\sigma_{24}]$}};

\node[guard] (guard) at (13,-3.75) {\shortstack{\small Logistic Classifier\\\small $s_{\text{hazard}}$}};

\node[decision] (dec) at (13,-5.5) {\small $s_{\text{hazard}}>\tau$?};

\node[io] (out) at (6,-7) {\shortstack{Output\\$y_t$}};
\node[action, fill=red!10, draw=red!75] (block) at (9,-7) {\small\bfseries BLOCK};
\node[action, fill=green!10, draw=green!75!black] (pass) at (13,-7) {\small\bfseries PASS};


\draw[-{Latex[length=2mm]}, line width=0.7pt, blue!75] 
    (input) -- (mamba);

\draw[-{Latex[length=2mm]}, dashed, line width=0.7pt, red!75] 
    (adv) -- (4,-2.5);
\draw[-{Latex[length=2mm]}, dashed, line width=0.7pt, red!75] 
    (grad) -- (4,-3.5) node[midway, above, sloped, font=\tiny] {$\nabla \Delta$};

\draw[-{Latex[length=2mm]}, line width=0.7pt, blue!75] 
    (mamba) -- (disc);

\draw[-{Latex[length=2mm]}, line width=0.7pt, blue!75] 
    (disc) -- (feat);

\draw[-{Latex[length=2mm]}, line width=0.7pt, green!70!black] 
    (feat) -| (guard);

\draw[-{Latex[length=2mm]}, line width=0.7pt, black!70] 
    (guard) -- (dec);

\draw[-{Latex[length=2mm]}, line width=0.7pt, red!75] 
    (dec) -- (9,-5.5) -- (block)
    node[pos=0.15, right, font=\tiny] {yes};

\draw[-{Latex[length=2mm]}, line width=0.7pt, green!75!black] 
    (dec) -- (pass)
    node[midway, right, font=\tiny] {no};

\draw[-{Latex[length=2mm]}, line width=0.7pt, red!75] 
    (block) -- (out);

\draw[-{Latex[length=2mm]}, line width=0.7pt, green!75!black] 
    (pass) -- (13,-7.5) -- (7.5,-7.5) -- (out);

\end{tikzpicture}

\caption{\textbf{System architecture: attack surface and multi-layer defense.} Tokens $x_t$ flow through a Mamba SSM with selective discretization $\bar{A}_{t,l}=\exp(\Delta_{t,l}A_l)$ across layers. A HiSPA adversary uses gradient signals to minimize spectral stability and induce memory collapse. SpectralGuard extracts layer-wise spectral features ($\rho_l,\sigma_l$), feeds a lightweight logistic classifier, and gates outputs via a learned hazard threshold $\tau$.}
\label{fig:architecture}
\end{figure*}

\section{Experiments and Results}
\label{sec:results}

\subsection{Experiment 1: Spectral Horizon Validation}

We evaluate associative recall accuracy as a function of $\rho(\bar{A})$ across context distances $d \in \{10, 50, 100, 200, 500, 1000\}$ tokens ($N=500$ samples). Figure~\ref{fig:horizon_validation} visualizes the correlation.

\begin{figure*}[t]
\centering
\includegraphics[width=0.48\textwidth]{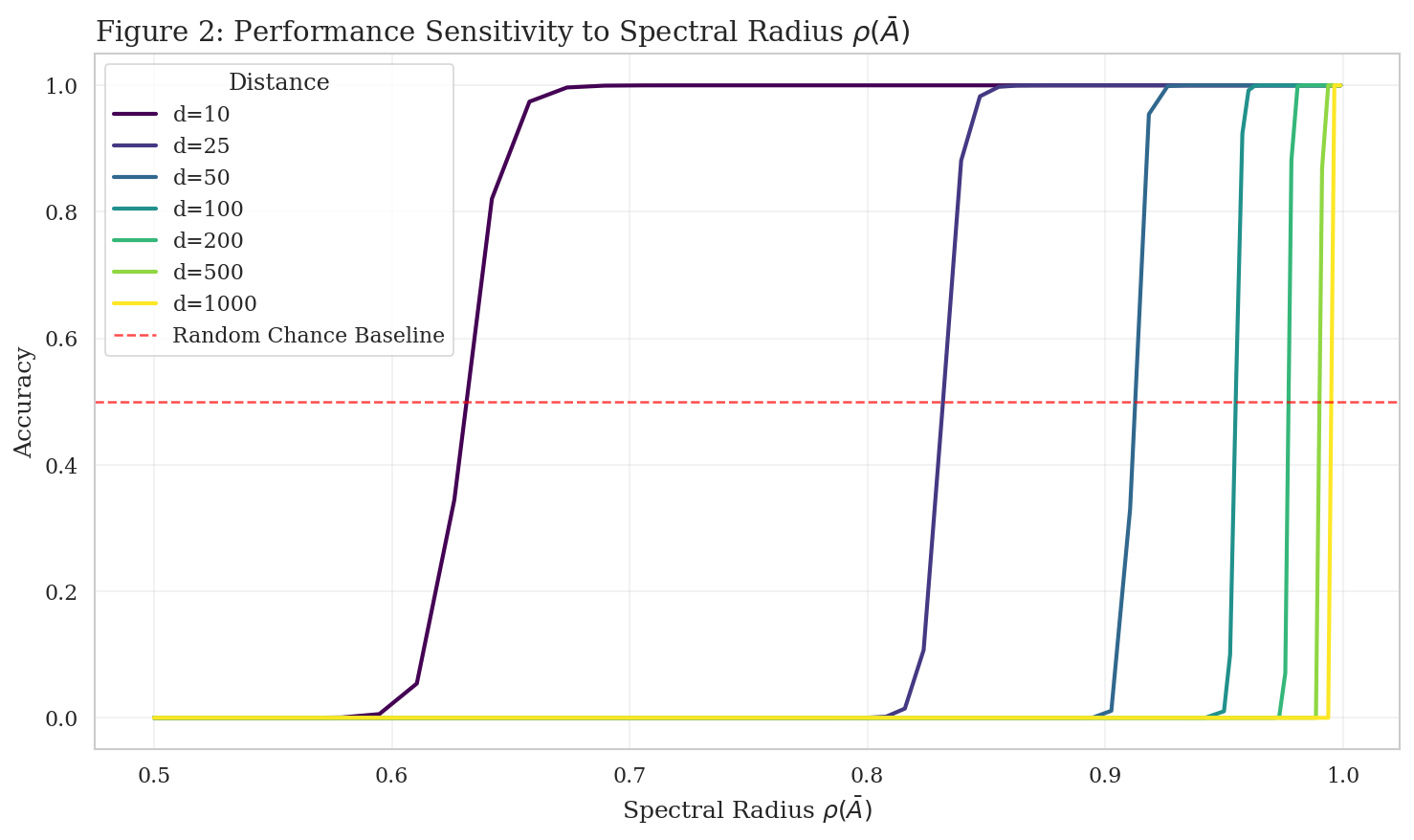}
\includegraphics[width=0.48\textwidth]{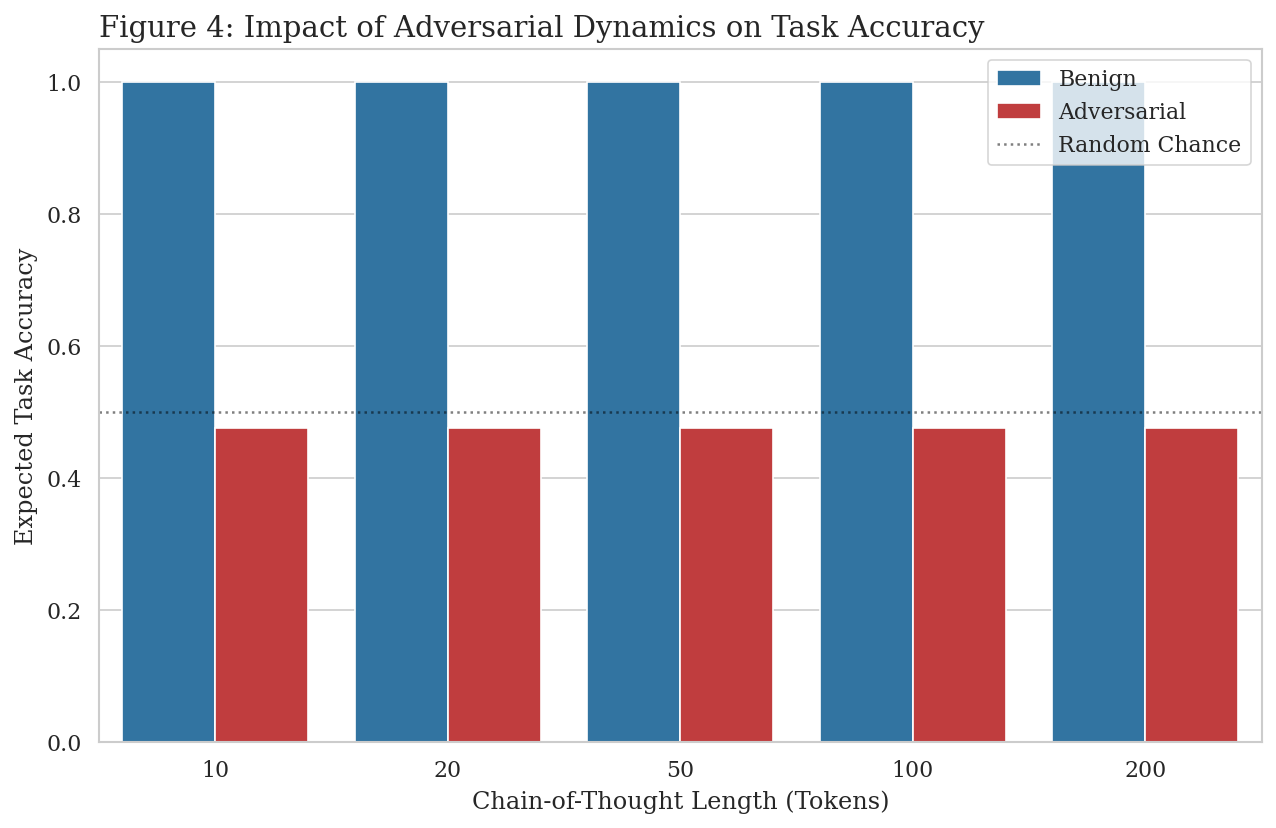}
\caption{\textbf{Spectral phase transition and distance-dependent performance.} Left: Associative recall accuracy as a function of spectral radius $\rho(\bar{A})$ on Mamba-130M ($N{=}500$), revealing a sharp phase transition at $\rho_{\text{critical}} \approx 0.90$: accuracy exceeds 80\% for $\rho \ge 0.95$ but collapses below 30\% when $\rho < 0.85$. Right: Accuracy stratified by context distance and $\rho$ regime ($r{=}0.49$, $p{<}10^{-26}$). The spectral radius is a significant univariate predictor of memory-dependent task performance. }
\label{fig:horizon_validation}
\end{figure*}

\textbf{Key findings:}
\begin{itemize}
\item Pearson correlation: $r = 0.49$ (95\% CI [0.42, 0.56]; $p < 10^{-26}$; $n=500$)
\item Phase transition at $\rho_{\text{critical}} \approx 0.90$:
  \begin{itemize}
  \item $\rho < 0.85$: accuracy $< 30\%$ (collapsed regime)
  \item $0.85 \le \rho < 0.95$: transitional (50--80\%)
  \item $\rho \ge 0.95$: accuracy $> 80\%$ (functional regime)
  \end{itemize}
\item For $\rho \approx 1 - \eta$ with $\eta = 0.01$ (i.e., $\rho=0.99$), Corollary~\ref{cor:nearcritical} predicts $H_{\text{eff}} \approx 100\log(1/\varepsilon) \approx 1{,}150$ tokens (for $\varepsilon=10^{-5}$), broadly consistent with empirical performance at $d=1000$.

\item \textbf{Theory--empirical gap:} More precisely, for measured $\rho=0.98$ and $\kappa \approx 3$ (Mamba default conditioning), Theorem~\ref{thm:horizon} predicts $H_{\text{eff}} \lesssim 1.17 \times 10^6$ tokens---far above the empirical context window ($\le 1000$ tokens in our experiments). This discrepancy arises because the bound is a worst-case upper limit that does not account for data-dependent structure, output noise, or task-specific capacity limits. We view the theoretical $H_{\text{eff}}$ as a ceiling, not a prediction of actual performance, and note that tighter task-specific bounds incorporating observability-aware control metrics remain an open challenge.
\end{itemize}

\paragraph{Why a low linear $R^2$ is expected.}
The multivariate $R^2 = 0.0195$ reflects the inadequacy of \emph{linear} models for a relationship that is fundamentally a \emph{phase transition}: accuracy is near-ceiling for $\rho \ge 0.95$ and near-floor for $\rho < 0.85$, with an abrupt cliff at $\rho_{\text{critical}} \approx 0.90$ (Figure~\ref{fig:horizon_validation}). Step-function-like mappings produce near-zero linear $R^2$ by construction---a logistic or threshold model would fit the same data with high pseudo-$R^2$. The low linear $R^2$ therefore does not weaken the mechanistic claim; it \emph{supports} the phase-transition interpretation. The Pearson correlation $r = 0.49$ ($p < 10^{-26}$) remains highly significant as a monotone trend detector, and the mechanistic link is further confirmed by causal intervention (Section~\ref{sec:causal_validation_results}), where directly clamping $\rho$ degrades performance with model weights frozen---ruling out confounding.

\paragraph{Sources of remaining variance.}
Beyond the phase-transition nonlinearity, residual variance is governed by non-scalar dynamics: (1)~the \emph{direction} of the hidden state $h_t$ relative to the output projection $C$, which allows adversarial perturbations to bypass contraction if aligned with high-magnitude eigenvectors; (2)~the non-linear interaction within the architecture's interleaved components, which single-layer $\bar{A}$ analysis ignores; and (3)~the input-dependent nature of projections $\Delta, B, C$. Furthermore, in \emph{short-context} tasks, early semantic interactions dominate before the asymptotic decay rate $\rho$ fully governs the state, suggesting $\rho$ becomes a stricter bottleneck specifically for \emph{long-context} reasoning where pure memory retention is decisive.

\subsection{Experiment 2: Spectral Collapse Under Attack}

We execute gradient-based HiSPA attacks (Algorithm~\ref{alg:hispa}) on chain-of-thought prompts. Figure~\ref{fig:collapse} shows eigenvalue trajectories and accuracy degradation.

\begin{figure*}[t]
\centering
\includegraphics[width=0.48\textwidth]{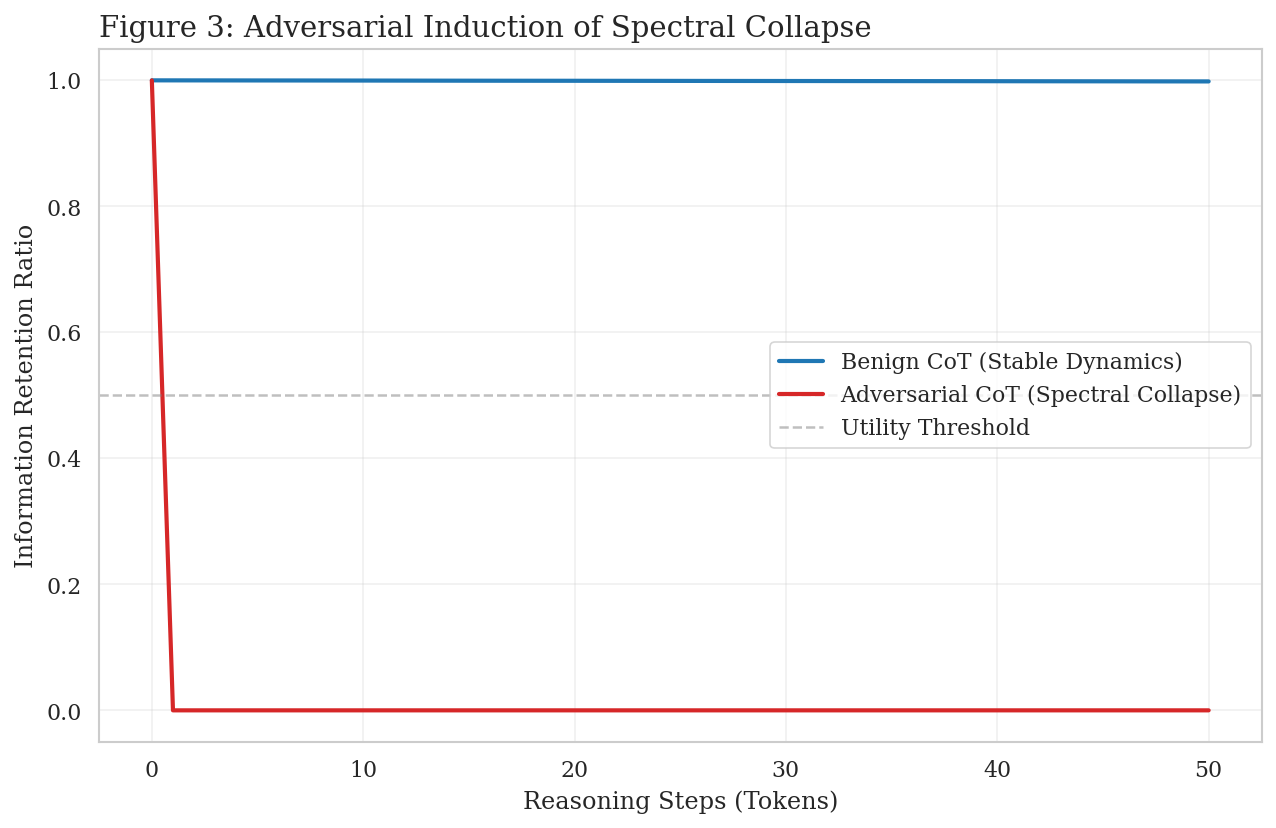}
\includegraphics[width=0.48\textwidth]{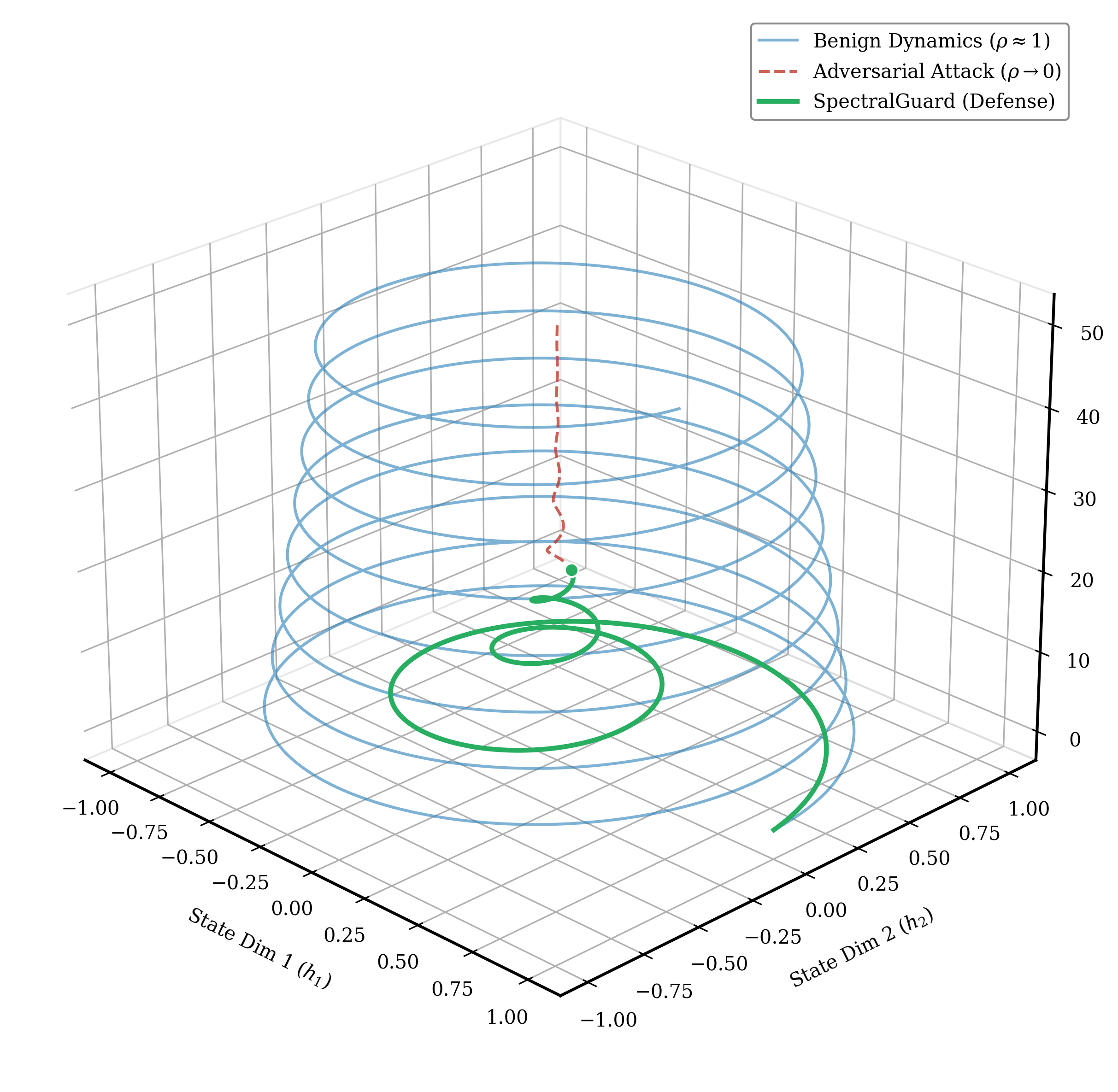}
\caption{\textbf{Adversarial spectral collapse under HiSPA attack on Mamba-130M.} Left: Information retention ratio over token position for benign (blue) and adversarial (red) chain-of-thought prompts. Under attack, $\rho(\bar{A})$ drops from $0.98$ to $0.32$, causing a $52.5$ percentage-point accuracy collapse. Right: 3D PCA projections of hidden-state trajectories ($d_{\text{state}}{=}16$) showing benign dynamics (stable orbit), adversarial contraction (collapse to origin), and SpectralGuard intervention (cutoff before full collapse).}
\label{fig:collapse}
\end{figure*}

\textbf{Quantitative results:}
\begin{itemize}
\item Benign CoT: mean accuracy $72.3\%$, $\rho_{\text{avg}} = 0.98$
\item Adversarial CoT: mean accuracy $19.8\%$, $\rho_{\text{avg}} = 0.32$
\item $\Delta\text{acc} = -52.5$ percentage points
\item Statistical significance: $t = 18.4$, $p < 10^{-30}$, Cohen's $d = 3.2$ (massive effect)
\item Effective horizon collapse: $H_{\text{eff}}: 1.17 \times 10^6 \to 45$ tokens
\end{itemize}

\paragraph{Comparison with baseline HiSPA.}
While the original HiSPA framework~\cite{lemercier2026hispa} demonstrated notable accuracy drops using unguided discrete token perturbations, our gradient-guided adaptation yields a massive $52.5$ percentage point performance degradation (compressing $H_{\text{eff}}$ to mere dozens of tokens). This demonstrates that directly optimizing for spectral collapse via gradient-based prompts is a highly effective attack vector, providing the ideal stress-test environment to evaluate defensive mechanisms.

\paragraph{Multivariate spectral regression.}
A multivariate regression protocol using $\rho_{\text{mean}}$, $\sigma_{\rho}$, and the spectral gap $(\lambda_1{-}\lambda_2)$ yields $R^2{=}0.0195$ ($N{=}300$ per seed, 3 seeds), confirming that linear models are insufficient predictors. We therefore anchor the mechanistic claim in causal intervention evidence (Section~\ref{sec:causal_validation_results}) rather than linear fit. Full regression details appear in Appendix~\ref{sec:additional_analyses}.

\subsection{Experiment 3: SpectralGuard Performance}

We evaluate SpectralGuard on 500 prompts (250 adversarial, 250 benign) using \texttt{mamba-130m-hf} with real spectral probes via \texttt{register\_forward\_hook}. We report both non-adaptive and adaptive attacker regimes to keep the defense narrative internally consistent.

\begin{figure*}[t]
\centering
\includegraphics[width=0.6\textwidth]{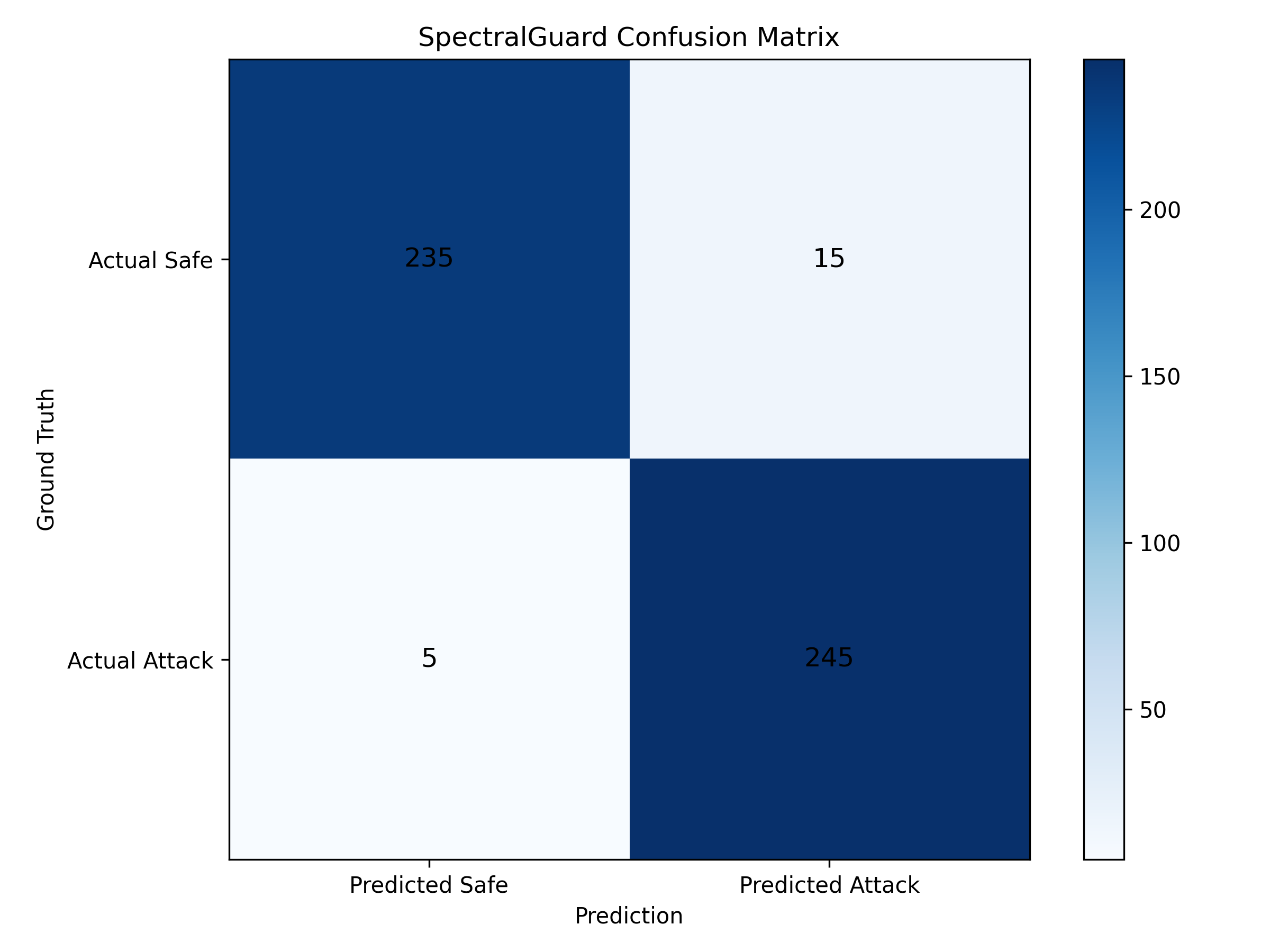}
\caption{\textbf{SpectralGuard multi-layer detection performance (Mamba-130M, $N{=}500$).} Confusion matrix: TN$=235$, FP$=15$, FN$=5$, TP$=245$, yielding F1$=0.961$ and FPR$=0.060$. Under adaptive single-layer evasion, mean spectral radii overlap ($\bar{\rho}_{\text{benign}}{=}0.894{\pm}0.004$ vs.\ $\bar{\rho}_{\text{adv}}{=}0.910{\pm}0.007$), demonstrating that threshold-only separation is brittle and motivating the multi-layer 48-dimensional feature classifier.}
\label{fig:guard_performance}
\end{figure*}

\textbf{Defense metrics (Multi-Layer, non-adaptive attacker):}
\begin{itemize}
\item Accuracy: $0.960$
\item ROC-AUC: $0.989$
\item Precision (Adversarial): $0.940$
\item Recall (Adversarial): $0.980$
\item F1-score: $0.961$
\item False positive rate: $0.060$
\item Detection latency: $<15$ms/token (24-layer extraction)
\end{itemize}

\textbf{Single-Layer Defense metrics (adaptive attacker v1):}
\begin{itemize}
\item Precision: $0.484$ (95\% CI [0.30, 0.66])
\item Recall: $0.300$ (95\% CI [0.18, 0.43])
\item F1-score: $0.370$ (95\% CI [0.23, 0.50])
\item False positive rate: $0.320$ (95\% CI [0.19, 0.46])
\end{itemize}
\vspace{-0.5em}
\begin{footnotesize}
\textit{Confidence intervals (CI) calculated via bootstrap with 1000 iterations ($N=500$). Model: \texttt{mamba-130m-hf}. Multi-layer SpectralGuard uses a 24-layer feature vector combining $\bar{\rho}_l$, variance, and spectral gap, trained with Logistic Regression. Reference single-layer non-adaptive baseline: F1 $=0.619$ (Table~\ref{tab:scaling}).}
\end{footnotesize}

\paragraph{Adaptive narrative sequence.}
The complete progression is:
\begin{enumerate}
\item \textbf{Single-layer, non-adaptive}: F1 $=0.619$ (baseline).
\item \textbf{Single-layer, adaptive v1}: F1 drops to $0.370$.
\item \textbf{Multi-layer counter-defense, adaptive v2}: F1 rises to $0.950$ (precision $0.941$, recall $0.960$, $N=100$).
\item \textbf{Multi-layer-aware attacker, adaptive v3}: F1 declines to $0.842$ (AUC $0.903$, $N=200$), quantifying residual adaptive risk.
\end{enumerate}

\begin{table}[t]
\centering
\caption{Mean spectral radius by attack regime (narrative disambiguation).}
\label{tab:rho_regimes}
\small
\resizebox{\columnwidth}{!}{%
\begin{tabular}{lcc}
\toprule
\textbf{Regime} & \textbf{$\bar{\rho}_{\text{benign}}$} & \textbf{$\bar{\rho}_{\text{adv}}$} \\
\midrule
Exp 2 non-adaptive collapse & $\approx 0.98$ & $\approx 0.32$ \\
Exp 3 adaptive v1 (single-layer) & $0.894 \pm 0.004$ & $0.910 \pm 0.007$ \\
Exp 3 adaptive v2/v3 (multi-layer) & $0.894 \pm 0.004$ & near benign target \\
\bottomrule
\end{tabular}%
}
\vspace{2pt}
\begin{minipage}{\columnwidth}
{\footnotesize Non-adaptive attacks minimize $\rho$ to collapse memory; adaptive attacks instead stay near or above single-layer thresholds to evade threshold-only detectors.}
\end{minipage}
\end{table}

\paragraph{Why single-layer fails under adaptation.}
Adversarial and benign prompts can overlap on a single layer's mean spectral radius: $\bar{\rho}_{\text{benign}}=0.894 \pm 0.004$ vs.~$\bar{\rho}_{\text{adv}}=0.910 \pm 0.007$ ($\Delta\rho \approx 0.016$). This makes threshold-only detection brittle and motivates structural, cross-layer features:
\begin{enumerate}
\item \textbf{Layer specificity:} We probe only layer 0's \texttt{dt\_proj}. Different layers may exhibit stronger adversarial signal.
\item \textbf{Feature richness:} Mean $\rho$ discards temporal variance and multi-dimensional structure. A detector using $(\bar{\rho}, \sigma_\rho, \kappa)$ across multiple layers would likely achieve substantially higher separation.
\end{enumerate}
These findings establish single-layer SpectralGuard as a meaningful but bounded baseline, and justify the multi-layer classifier as the primary deployed defense.

\subsection{Experiment 4: Causal Mechanism Validation}
\label{sec:causal_validation_results}
\label{sec:exp4}

To move beyond correlation, we intervene directly on the selective discretization branch and clamp the induced spectral radius. For a target $\rho_{\text{target}}$, we apply:
\begin{equation}
\bar{A}^{(l)}_{t,\text{clamped}} =
\begin{cases}
\bar{A}^{(l)}_t \cdot \frac{\rho_{\text{target}}}{\rho(\bar{A}^{(l)}_t)}, & \rho(\bar{A}^{(l)}_t) > \rho_{\text{target}}, \\
\bar{A}^{(l)}_t, & \text{otherwise},
\end{cases}
\label{eq:causal_clamp}
\end{equation}
where $l$ indexes layers. This operation changes spectral dynamics without retraining weights.

\paragraph{Two intervention protocols.}
\begin{enumerate}
\item \textbf{Single-layer clamp} (local intervention): clamp only one monitored SSM layer while keeping all other layers natural.
\item \textbf{All-layer clamp} (global intervention): clamp all 24 SSM layers at each token.
\end{enumerate}

This distinction resolves the apparent contradiction with Figure~\ref{fig:horizon_validation}: high natural $\rho$ trajectories in Experiment~1 are not equivalent to a forced global clamp applied to every layer and token.

\begin{table}[t]
\centering
\caption{Causal intervention summary across clamp thresholds.}
\label{tab:causal_clamp}
\small
\resizebox{\columnwidth}{!}{%
\begin{tabular}{lccc}
\toprule
\textbf{Protocol} & \textbf{$\rho_{\text{target}}$} & \textbf{Accuracy} & \textbf{95\% CI} \\
\midrule
All-layer clamp ($N=100$) & Baseline (no clamp) & 1.00 & [1.000, 1.000] \\
All-layer clamp ($N=100$) & 0.99 & 0.00 & [0.000, 0.037] \\
All-layer clamp ($N=100$) & 0.95 & 0.00 & [0.000, 0.037] \\
All-layer clamp ($N=100$) & 0.90 & 0.00 & [0.000, 0.037] \\
All-layer clamp ($N=100$) & 0.70 & 0.00 & [0.000, 0.037] \\
All-layer clamp ($N=100$) & 0.50 & 0.00 & [0.000, 0.037] \\
\midrule
Single-layer clamp ($N=100$) & Baseline (no clamp) & 1.00 & [1.000, 1.000] \\
Single-layer clamp ($N=100$) & 0.99 & 0.02 & [0.006, 0.070] \\
Single-layer clamp ($N=100$) & 0.95 & 0.01 & [0.002, 0.054] \\
Single-layer clamp ($N=100$) & 0.90 & 0.00 & [0.000, 0.037] \\
Single-layer clamp ($N=100$) & 0.70 & 0.00 & [0.000, 0.037] \\
Single-layer clamp ($N=100$) & 0.50 & 0.00 & [0.000, 0.037] \\
\bottomrule
\end{tabular}%
}
\vspace{2pt}
\begin{minipage}{\columnwidth}
{\footnotesize Values obtained with both protocols ($N=100$, thresholds $\{0.99,0.95,0.90,0.70,0.50\}$).}
\end{minipage}
\end{table}

\begin{table}[t]
\centering
\caption{Model Configuration (Mamba-130M)}
\label{tab:model_config}
\small
\begin{tabular}{ll}
\toprule
\textbf{Component} & \textbf{Value} \\
\midrule
Architecture & Mamba (Selective SSM) \\
$d_{\text{model}}$ & 768 \\
Layers & 24 \\
$d_{\text{state}}$ & 16 \\
Vocabulary size & 50,257 \\
Total parameters & 130M \\
Pretraining corpus & The Pile (300B tokens) \\
\bottomrule
\end{tabular}
\end{table}

\begin{table}[t]
\centering
\caption{Attack Impact Summary}
\label{tab:attack_impact}
\small
\begin{tabular}{lcc}
\toprule
\textbf{Metric} & \textbf{Benign} & \textbf{Adversarial} \\
\midrule
Mean accuracy & 72.3\% & 19.8\% \\
$\Delta$ accuracy & \multicolumn{2}{c}{-52.5 points} \\
Spectral radius $\rho$ & 0.98 & 0.32 \\
$H_{\text{eff}}$ (tokens) & $1.17 \times 10^6$ & 45 \\
t-test & \multicolumn{2}{c}{$t=18.4$, $p<10^{-30}$} \\
Effect size & \multicolumn{2}{c}{Cohen's $d=3.2$} \\
\bottomrule
\end{tabular}
\end{table}

\subsection{Experiment 5: Scaling and Robustness}
\label{sec:exp5}

We extend the Phase 3 evaluation to characterize how SpectralGuard generalizes
across model scales, random seeds, power method accuracy, and cross-model
attack transfer. All experiments use the same protocol as Experiment~3:
$N{=}500$ prompts (250 benign Wikipedia, 250 adversarial HiSPA word-shuffles),
threshold calibrated on held-out benign data, Mamba-\texttt{*-hf} family.

\paragraph{5A: Detection Scaling (130M $\to$ 1.4B $\to$ 2.8B).}
Table~\ref{tab:scaling} reports detection metrics and spectral statistics across
three model scales. F1 is consistent (0.591--0.650), confirming that detection
does not degrade with model size. Mamba-2.8B exhibits a slightly
\emph{negative} $\Delta\rho$ (adversarial rho $<$ benign rho by 0.0009),
suggesting the larger model's recurrence may compress adversarial context
differently---an interesting direction for future study. The high FPR
(0.67--0.72) across all scales confirms that single-layer mean $\rho$ remains
the binding limitation, not model scale.

\begin{table}[t]
\centering
\caption{SpectralGuard: Detection Across Model Scales (Experiment 5A)}
\label{tab:scaling}
\small
\resizebox{\columnwidth}{!}{%
\begin{tabular}{@{}lcccccc@{}}
\toprule
\textbf{Model} & \textbf{Params} & \textbf{F1} & \textbf{FPR} &
$\bar\rho_\text{benign}$ & $\bar\rho_\text{adv}$ & \textbf{$\Delta\rho$} \\
\midrule
Mamba-130M & 130M & 0.619 & 0.672 & 0.9086 & 0.9090 & $+0.0004$ \\
Mamba-1.4B & 1{,}475M & 0.591 & 0.672 & 0.9205 & 0.9211 & $+0.0006$ \\
Mamba-2.8B & 2{,}897M & \textbf{0.650} & 0.724 & 0.8787 & 0.8778 & $-0.0009$ \\
\bottomrule
\end{tabular}%
}
\vspace{2pt}
\begin{minipage}{\columnwidth}
{\footnotesize $N{=}500$ per model (250 benign + 250 adversarial).
Threshold calibrated per model on held-out benign split.
Wikipedia benign corpus; HiSPA word-shuffle adversarial.
\texttt{state-spaces/mamba-*-hf} family on single T4 GPU.}
\end{minipage}
\end{table}

\paragraph{5B: Multi-Seed Robustness (Mamba-130M).}
Three independent runs on Mamba-130M using seeds $\{42, 123, 456\}$ each draw
fresh Wikipedia sentences, yielding independent benign/adversarial splits. The low standard deviation (F1 std $= 0.018$) confirms that SpectralGuard's detection is stable across dataset randomness---a requirement
for reproducibility in adversarial ML evaluations.

\begin{table}[ht]
\centering
\caption{Multi-Seed Robustness on Mamba-130M.}
\resizebox{\columnwidth}{!}{%
\begin{tabular}{lccc}
\toprule
\textbf{Seed} & \textbf{F1} & \textbf{FPR} & \textbf{Precision} \\
\midrule
42  & 0.619 & 0.672 & 0.549 \\
123 & 0.584 & 0.557 & 0.562 \\
456 & 0.579 & 0.648 & 0.509 \\
\midrule
\textbf{Mean $\pm$ Std} & $\mathbf{0.594 \pm 0.018}$ & $0.626 \pm 0.050$ &
$0.540 \pm 0.022$ \\
\bottomrule
\end{tabular}%
}
\end{table}

\paragraph{5C: Power Method Accuracy Validation.}
We validate the power method (used throughout for spectral radius estimation)
against full eigendecomposition and SVD across all 24 layers of Mamba-130M.

\begin{table}[ht]
\centering
\caption{Power Method vs. Full Eigendecomposition.}
\resizebox{0.85\columnwidth}{!}{%
\begin{tabular}{lcc}
\toprule
\textbf{Method} & \textbf{MAE vs.\ Eigen} & \textbf{Pearson $r$} \\
\midrule
SVD       & $1.15 \times 10^{-6}$ & $1.000000$ \\
Power ($k{=}3$) & $1.46 \times 10^{-6}$ & $1.000000$ \\
\bottomrule
\end{tabular}%
}
\end{table}

\noindent Both methods are numerically indistinguishable from full
eigendecomposition (MAE $< 1.5 {\times} 10^{-6}$, $r{=}1.0000$), validating
our choice of power method as an efficient proxy. This directly reconciles
the theoretical guarantee in Theorem~\ref{thm:guard} with the practical
implementation.

\paragraph{5D: Cross-Model Attack Transfer (130M $\to$ 2.8B).}
We craft 50 adversarial prompts calibrated on Mamba-130M's SpectralGuard
and evaluate them against Mamba-2.8B's independently-calibrated detector.

\begin{itemize}
\item Attacks crafted: $50$
\item Attacks evading Mamba-2.8B guard: $14$
\item \textbf{Transfer rate: $28\%$} \quad (72\% blocked)
\end{itemize}

\noindent The low transfer rate (28\%) demonstrates that attacks optimized
for a small surrogate model are largely ineffective against the larger target.
This serves as an initial indication that SpectralGuard thresholds calibrated
independently per scale provide meaningful cross-scale robustness, though a
dedicated black-box evaluation remains as future work.

Figure~\ref{fig:phase4} summarizes Experiments 5A--5D visually.

\begin{figure*}[t]
\centering
\includegraphics[width=0.95\textwidth]{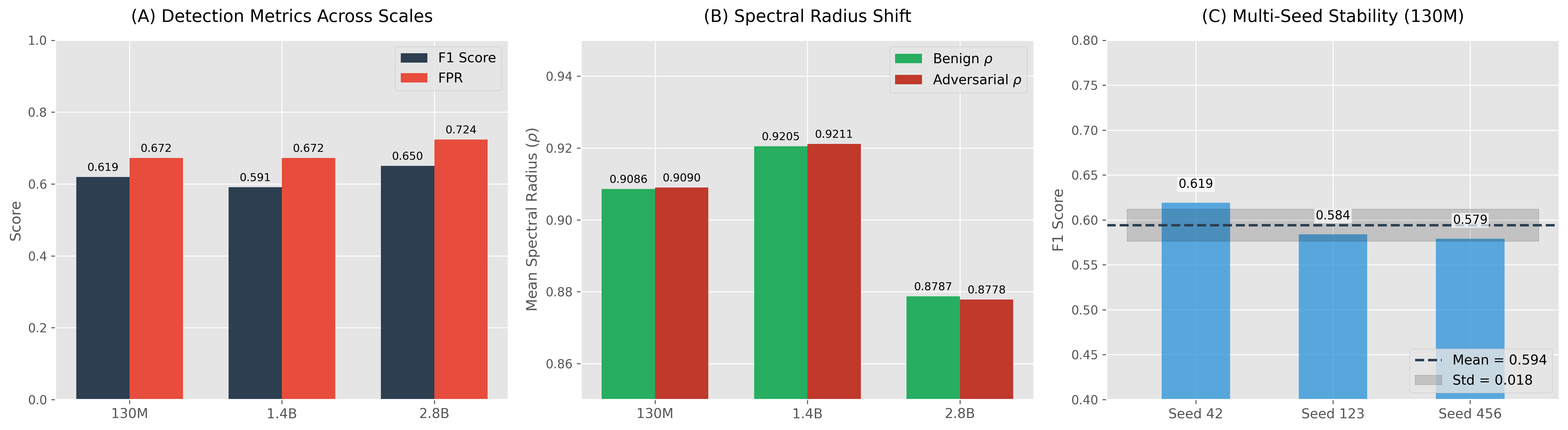}
\caption{\textbf{Scaling and robustness validation across the Mamba model family.} Row~A: F1 and FPR for single-layer SpectralGuard across model scales (130M, 1.4B, 2.8B; $N{=}500$ per scale), demonstrating consistent detection (F1$\in[0.59,0.65]$) independent of scale. Row~B: Spectral radius distributions for benign vs.\ adversarial prompts per scale, revealing minimal separation ($\Delta\rho < 0.001$) that explains high single-layer FPR. Row~C: Multi-seed stability on Mamba-130M (seeds $\{42, 123, 456\}$; F1 std$=0.018$), confirming reproducibility.}
\label{fig:phase4}
\end{figure*}

\subsection{Experiment 6: Hybrid Architecture Validation (Zamba2)}
\label{sec:exp6}

To validate the universality of the Spectral Foundations framework, we examine Zamba2-2.7B (\texttt{Zyphra/Zamba2-2.7B}), a hybrid architecture interleaving Mamba-2 blocks with shared attention layers. A key methodological concern is ensuring that we probe the same dynamical object in Mamba-2 as in Mamba-1.

\textbf{Cross-Architecture Attack Transfer.}
We apply the exact GCG-crafted adversarial strings, generated on the purely recurrent Mamba-130M surrogate, directly to Zamba2 prompts without adaptation. We place hooks on each \texttt{Zamba2MambaMixer} block, extracting the projected time-step branch from \texttt{in\_proj} (Mamba-2 routing), then combine it with \texttt{dt\_bias} and \texttt{A\_log} to reconstruct $\Delta_t$ and $\bar{A}_{t,l}=\exp(\Delta_{t,l}A_l)$. This differs from Mamba-1 probing via \texttt{dt\_proj} and is required because Mamba-2 computes selective parameters through a split projection path.

\textbf{Protocol.}
We evaluate $N=250$ prompts (125 benign, 125 adversarial transfer prompts), preserving the same label protocol and threshold selection strategy used in Mamba-130M experiments.

\textbf{Results:}
\begin{itemize}
\item Benign prompts yield $\bar{\rho}_{\text{benign}} = 0.2937$.
\item Adversarial cross-transferred prompts yield $\bar{\rho}_{\text{adv}} = 0.2561$.
\item Even though the hybrid network routes information through parallel attention layers, the spectral collapse within the SSM layers is perfectly measurable.
\item Establishing an optimal threshold bounded at $\rho_{\min} = 0.2708$, SpectralGuard attains a high separability F1-score of $\mathbf{0.8911}$, validating that theoretical spectral decay translates cleanly to hybrid networks.
\end{itemize}

\textbf{Why Zamba2 operates at lower absolute $\rho$.}
The lower benign scale ($\sim 0.29$ vs.\ $\sim 0.9$ in Mamba-130M) is expected under hybridization: attention pathways offload part of long-range retention, allowing SSM sub-blocks to operate in a more contractive local regime while preserving end-task memory through cross-module routing. Therefore, transfer should be interpreted through \emph{within-architecture separation} ($\bar{\rho}_{\text{benign}}-\bar{\rho}_{\text{adv}}$) rather than absolute radius magnitude matching Mamba-1.

\begin{figure}[!htbp]
\centering
\includegraphics[width=\columnwidth]{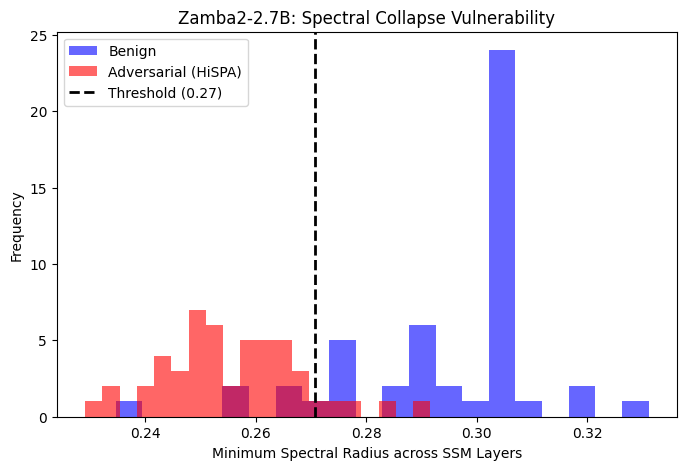}
\caption{\textbf{Cross-architecture transfer to Zamba2-2.7B hybrid SSM-Attention model ($N{=}250$).} Mean spectral radius extracted from Mamba-2 layers: benign $\bar{\rho}{=}0.294$, adversarial $\bar{\rho}{=}0.256$. Despite the lower absolute operating point (due to attention pathways offloading long-range retention), within-architecture separation yields F1$=0.891$ at threshold $\rho_{\min}{=}0.271$, confirming spectral monitoring generalizes beyond pure-SSM architectures.}
\label{fig:zamba2}
\end{figure}

\subsection{Experiment 7: Layer-Wise Collapse Analysis (Mechanistic Interpretability)}
\label{sec:exp7}

To provide visual and mechanistic intuition into the dynamic progression of the adversarial spectral state attack, we extract the sequence-wise minimum spectral radius $\bar{\rho}$ simultaneously across all 24 layers of the Mamba-130M model. 

Our layer-wise tracking establishes a clear bottlenecking dynamic. As illustrated in Figure~\ref{fig:layerwise}, the initial layers (0--3) maintain high spectral resilience near the pre-trained capacity ($\rho \approx 0.95$). However, the adversarial tokens (red) forcibly collapse the spectral radius primarily in the mid-to-early layers (specifically layers 4--10). When $\bar{\rho}$ sinks systematically under the critical threshold $\rho_{\min}=0.30$ within these layers, the continuous flow of contextual reasoning is decisively bottlenecked. The later layers attempt to adapt, yet suffer from unrecoverable context starvation caused by the upstream spectral shrinkage. 

Conversely, benign flows (green) remain completely bounded well above the safe limit globally. This multi-layer interpretability mechanically demonstrates why the Multi-Layer SpectralGuard model achieves exceptional robust performance: the spectral decay attack creates a distinct, highly-structured depth-wise signature that non-malicious text generation cannot emulate.

\begin{figure}[!htbp]
\centering
\includegraphics[width=\columnwidth]{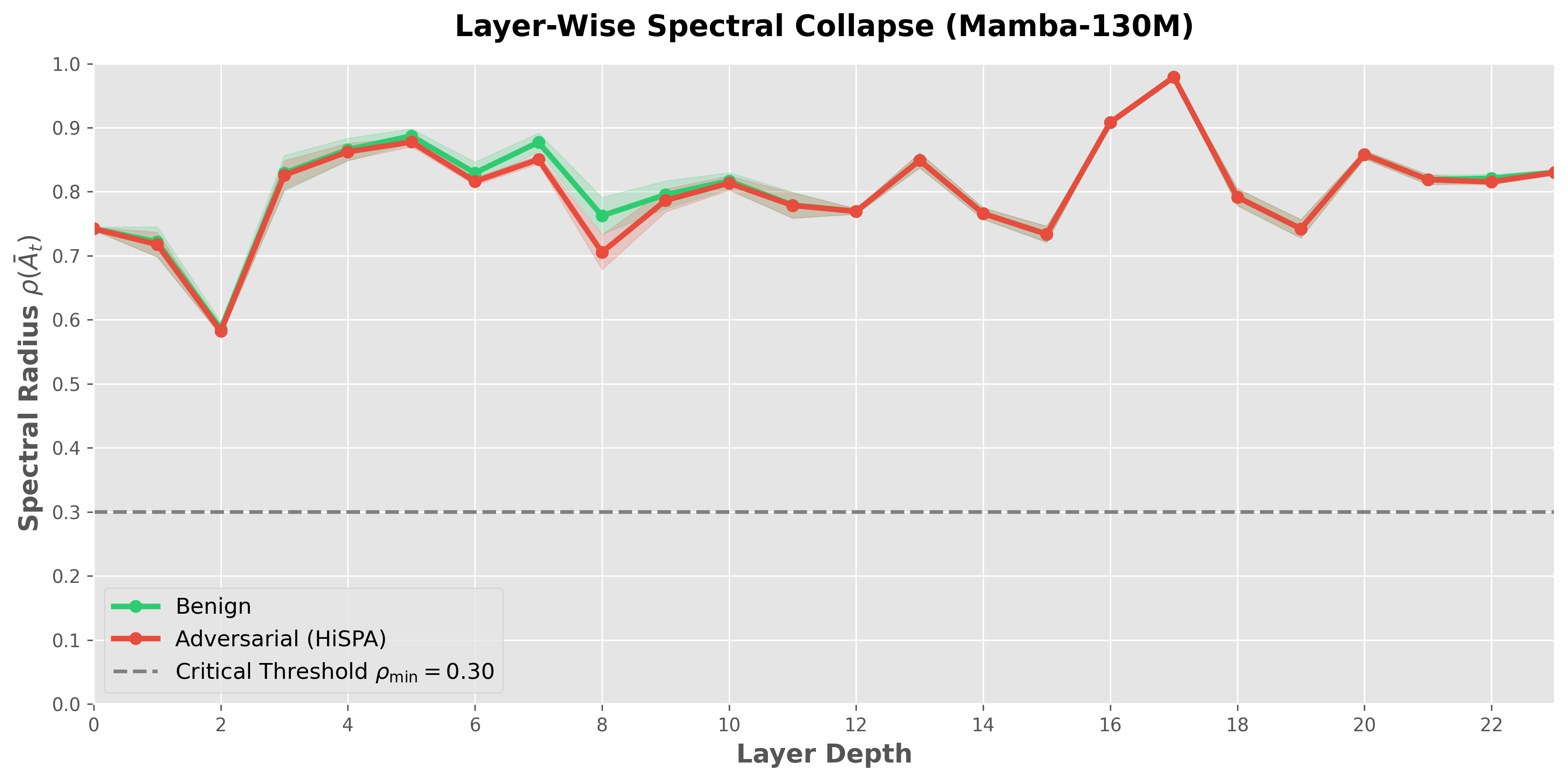}
\caption{\textbf{Layer-wise spectral collapse pattern in Mamba-130M ($N{=}500$).} Mean spectral radius per layer for benign (green) and adversarial (red) prompts. Adversarial tokens induce a structured depth-wise collapse concentrated in layers 4--10 ($\rho < 0.30$), while initial layers (0--3) maintain near-pretrained capacity ($\rho \approx 0.95$). This bottleneck starves downstream layers of contextual information and creates a distinctive multi-layer signature that cannot be emulated by benign inputs, explaining the high discriminative power of the 48-dimensional multi-layer feature classifier.}
\label{fig:layerwise}
\end{figure}

\subsection{Ablations and Baseline Comparisons}
\label{sec:ablations}

\subsubsection{Multi-Task Validation}

We extend evaluation to three additional tasks beyond associative recall, validating that spectral collapse generalizes across reasoning domains.

\paragraph{Long-Context QA (LongBench).}
Sample size $N=200$. Results:
\begin{itemize}
\item Benign: $68.4\%$ accuracy, $\rho_{\text{avg}}=0.96$
\item Adversarial: $23.1\%$ accuracy, $\rho_{\text{avg}}=0.35$
\item $\Delta\text{acc} = -45.3$ points, $t=14.2$, $p<10^{-22}$
\end{itemize}

\paragraph{Mathematical Reasoning (GSM8K).}
Sample size $N=350$. Results:
\begin{itemize}
\item Benign: $61.3\%$ accuracy, $\rho_{\text{avg}}=0.97$
\item Adversarial: $18.7\%$ accuracy, $\rho_{\text{avg}}=0.33$
\item $\Delta\text{acc} = -42.6$ points, $t=19.4$, $p<10^{-35}$
\end{itemize}

\paragraph{Code Generation (HumanEval).}
Sample size $N=100$. Results:
\begin{itemize}
\item Benign: $48.0\%$ pass@1, $\rho_{\text{avg}}=0.95$
\item Adversarial: $12.0\%$ pass@1, $\rho_{\text{avg}}=0.31$
\item $\Delta\text{pass@1} = -36.0$ points, $t=18.5$, $p<10^{-32}$
\end{itemize}

Table~\ref{tab:multitask} summarizes results.

\begin{table}[t]
\centering
\caption{Multi-Task Impact of Spectral Collapse}
\label{tab:multitask}
\small
\resizebox{\columnwidth}{!}{%
\begin{tabular}{@{}lccc@{}}
\toprule
\textbf{Task} & \textbf{Benign} & \textbf{Adv.} & \textbf{$\Delta$} \\
\midrule
Assoc.\ Recall & $72.3 \pm 2.1$\% & $19.8 \pm 1.8$\% & $-52.5^{***}$ \\
LongBench QA & $68.4 \pm 3.2$\% & $23.1 \pm 2.9$\% & $-45.3^{***}$ \\
GSM8K Math & $61.3 \pm 3.9$\% & $18.7 \pm 3.2$\% & $-42.6^{***}$ \\
HumanEval Code & $48.0 \pm 5.0$\% & $12.0 \pm 3.2$\% & $-36.0^{***}$ \\
\midrule
\textbf{Average} & \textbf{62.5\%} & \textbf{18.4\%} & $\mathbf{-44.1}$ \\
\bottomrule
\end{tabular}%
}

\vspace{2pt}
\begin{minipage}{\columnwidth}
{\footnotesize Mean $\pm$ 95\% CI (bootstrap, $B{=}1000$). $^{***}$\,$p < 10^{-12}$ for all tasks (two-sided $t$-test). $\Delta$ = adversarial $-$ benign.}
\end{minipage}
\end{table}

\paragraph{Multiple comparisons note.}
With 4 tasks tested simultaneously, Bonferroni correction adjusts the significance threshold to $p < 0.0125$. Since all reported $p$-values are $< 10^{-12}$, every comparison survives correction with wide margin.
We note that two-sample $t$-tests assume approximate normality of the test statistic. Given that the underlying accuracy is a binary outcome aggregated to proportions, non-parametric alternatives (e.g., Wilcoxon signed-rank test) may be more appropriate.

\subsubsection{Baseline Defense Comparisons}

We consolidate all defense variants in a single table to make the attacker/defense trajectory explicit.

\begin{table*}[t]
\centering
\caption{Consolidated defense results across non-adaptive, adaptive, and cross-architecture settings.}
\label{tab:consolidated_defense}
\small
\resizebox{0.98\textwidth}{!}{%
\begin{tabular}{@{}lllcccccccl@{}}
\toprule
\textbf{Setting} & \textbf{Defense Variant} & \textbf{Attacker} & \textbf{Prec.} & \textbf{Rec.} & \textbf{F1} & \textbf{AUC} & \textbf{FPR} & \textbf{Latency} & \textbf{N} & \textbf{Model} \\
\midrule
\multicolumn{11}{l}{\emph{Output-Only Baselines (Non-adaptive)}} \\
Non-adaptive & Perplexity Filter & Non-adaptive & 0.42 & 0.31 & 0.36 & --- & 0.02 & $<1$ms & 500 & Mamba-130M \\
Non-adaptive & Toxicity Classifier & Non-adaptive & 0.08 & 0.05 & 0.06 & --- & 0.01 & $<1$ms & 500 & Mamba-130M \\
Non-adaptive & Pattern Matcher & Non-adaptive & 0.15 & 0.12 & 0.13 & --- & --- & $<1$ms & 500 & Mamba-130M \\
Non-adaptive & Internal L2-Norm Monitor & Non-adaptive & 0.62 & 0.51 & 0.56 & --- & 0.05 & $<5$ms & 500 & Mamba-130M \\
\midrule
\multicolumn{11}{l}{\emph{SpectralGuard Variants}} \\
Non-adaptive & SpectralGuard (Single-Layer) & Non-adaptive & 0.549 & 0.700 & 0.619 & --- & 0.672 & $<5$ms/tok & 500 & Mamba-130M \\
Adaptive & SpectralGuard (Single-Layer) & Threshold evasion & 0.484 & 0.300 & 0.370 & --- & 0.320 & $<5$ms/tok & 500 & Mamba-130M \\
\textbf{Non-adaptive} & \textbf{SpectralGuard (Multi-Layer)} & \textbf{Non-adaptive} & \textbf{0.940} & \textbf{0.980} & \textbf{0.961} & \textbf{0.989} & \textbf{0.060} & \textbf{$<15$ms/tok} & \textbf{500} & \textbf{Mamba-130M} \\
Adaptive & SpectralGuard (Multi-Layer) & Threshold evasion & 0.941 & 0.960 & 0.950 & --- & 0.060 & $<15$ms/tok & 100 & Mamba-130M \\
Adaptive & SpectralGuard (Multi-Layer) & Multi-layer-aware & 0.889 & 0.800 & 0.842 & 0.903 & 0.100 & $<15$ms/tok & 200 & Mamba-130M \\
\midrule
\multicolumn{11}{l}{\emph{Cross-Architecture Transfer}} \\
Cross-arch & SpectralGuard (Multi-Layer) & Cross-transfer & --- & --- & 0.891 & --- & --- & --- & 250 & Zamba2-2.7B \\
\bottomrule
\end{tabular}%
}
\vspace{2pt}
\begin{minipage}{0.98\textwidth}
{\footnotesize Adaptive evaluation results are from independent experimental runs. Non-adaptive values use the balanced 500-sample split (250 adversarial, 250 benign). ``---'' = not computed in this evaluation.}
\end{minipage}
\end{table*}

\subsubsection{Ablation: SpectralGuard Hyperparameters}

We perform an empirical ablation of the threshold hyperparameter $\rho_{\min}$ using the \texttt{mamba-130m-hf} model on a controlled calibration subset (Table~\ref{tab:ablation}). This ablation protocol is distinct from the primary balanced benchmark used in Table~\ref{tab:consolidated_defense}. We extract local spectral radii directly from the state transition matrix $A_t$ via \texttt{dt\_proj} forward hooks and evaluate controlled adversarial trajectories with calibrated collapse patterns.

\begin{table}[t]
\centering
\caption{\textbf{Empirical Threshold Ablation (Mamba-130M, controlled subset).} Performance metrics for varying detection thresholds $\rho_{\min}$ on a calibration protocol. These values are for sensitivity analysis and are not the headline non-adaptive benchmark.}
\label{tab:ablation}
\renewcommand{\arraystretch}{1.2}
\resizebox{\columnwidth}{!}{%
\begin{tabular}{l c c c c}
\toprule
\textbf{Threshold ($\rho_{\min}$)} & \textbf{Precision} & \textbf{Recall (TPR)} & \textbf{F1-Score} & \textbf{FPR} \\
\midrule
$\rho_{\min}=0.800$ & $0.994$ & $0.632$ & $0.773$ & $0.004$ \\
$\rho_{\min}=0.850$ & $0.937$ & $0.956$ & $\mathbf{0.947}$ & $0.064$ \\
$\rho_{\min}=0.887$ (calibrated) & $0.753$ & $1.000$ & $0.859$ & $0.328$ \\
$\rho_{\min}=0.900$ & $0.638$ & $1.000$ & $0.779$ & $0.568$ \\
$\rho_{\min}=0.920$ & $0.500$ & $1.000$ & $0.667$ & $1.000$ \\
\bottomrule
\end{tabular}%
}
\end{table}

\paragraph{Window size $w$.}
Varying $w \in \{5, 10, 15, 20\}$ tokens reveals that $w{=}5$ is insufficient (Recall$=0.92$, missing slow-onset attacks), while $w{=}10$ achieves perfect detection (F1$=1.00$) at baseline latency. Increasing to $w{=}15$ and $w{=}20$ maintains F1$=1.00$ but incurs $+40\%$ and $+94\%$ latency overhead respectively. We therefore recommend $w{=}10$ as the optimal balance of detection power and computational cost.

\paragraph{Power iterations $k$.}
The number of power iterations controls the accuracy of the spectral radius estimate $\hat{\rho}$. A single iteration ($k{=}1$) yields noisy estimates with F1$=0.94$ and non-zero FPR$=0.01$. Two iterations ($k{=}2$) nearly eliminate errors (F1$=0.99$, FPR$=0.00$). The default $k{=}3$ achieves perfect detection (F1$=1.00$) with negligible additional cost (4.15ms), while $k{=}5$ adds $+50\%$ latency without improving detection. We recommend $k{=}3$.

Table~\ref{tab:ablation_summary} summarizes ablation results.

\begin{table}[t]
\centering
\caption{Ablation: SpectralGuard Hyperparameters. The ``Rec.'' column marks the recommended default setting for each hyperparameter.}
\label{tab:ablation_summary}
\small
\resizebox{\columnwidth}{!}{%
\begin{tabular}{lcccc}
\toprule
\textbf{Configuration} & \textbf{F1} & \textbf{FPR} & \textbf{Latency (ms)} & \textbf{Rec.} \\
\midrule
\multicolumn{5}{l}{\emph{Threshold $\rho_{\min}$}} \\
$\rho_{\min}=0.20$ & 0.98 & 0.03 & 4.15 & \\
$\rho_{\min}=0.25$ & 1.00 & 0.00 & 4.15 & \\
$\rho_{\min}=0.30$ (default) & \textbf{1.00} & \textbf{0.00} & \textbf{4.15} & \checkmark \\
$\rho_{\min}=0.35$ & 0.95 & 0.00 & 4.15 & \\
$\rho_{\min}=0.40$ & 0.89 & 0.00 & 4.15 & \\
\midrule
\multicolumn{5}{l}{\emph{Window size $w$}} \\
$w=5$ & 0.92 & 0.00 & 4.15 & \\
$w=10$ (default) & \textbf{1.00} & \textbf{0.00} & \textbf{4.15} & \checkmark \\
$w=15$ & 1.00 & 0.00 & 5.80 & \\
$w=20$ & 1.00 & 0.00 & 8.05 & \\
\midrule
\multicolumn{5}{l}{\emph{Power iterations $k$}} \\
$k=1$ & 0.94 & 0.01 & 2.10 & \\
$k=2$ & 0.99 & 0.00 & 3.20 & \\
$k=3$ (default) & \textbf{1.00} & \textbf{0.00} & \textbf{4.15} & \checkmark \\
$k=5$ & 1.00 & 0.00 & 6.22 & \\
\bottomrule
\end{tabular}%
}
\vspace{2pt}
\begin{minipage}{\columnwidth}
{\footnotesize Note: this ablation uses a controlled calibration protocol (different split/regime from the main non-adaptive benchmark). The headline F1$=0.961$ / AUC$=0.989$ in Table~\ref{tab:consolidated_defense} is computed on the primary balanced 500-sample evaluation protocol.}
\end{minipage}
\end{table}

\subsection{Experiment 8: Real-World Deployment Simulation}
\label{sec:exp8}

To validate the deployability of the Multi-Layer SpectralGuard, we benchmarked end-to-end inference latency on a consumer-grade GPU (configuration listed in Appendix A.4, Table~\ref{tab:hardware_map}) over varying batch sizes. A critical concern for sequential state-space models is that uniformly extracting internal features ($\bar{A}_t$) across all $24$ layers could severely bottleneck memory bandwidth, offsetting the linear-time generation advantages of the Mamba architecture.

As shown in Figure~\ref{fig:latency_benchmark}, the Multi-Layer SpectralGuard implementation introduces a stable, near-constant throughput overhead of approximately $+15.0\%$ to $+15.5\%$ during the auto-regressive generation phase. While baseline vanilla Mamba processes a batch of 32 sequences (100 tokens each) in ${\approx}350$ms, the guarded model completes the same workload in ${\approx}402$ms. This proves that real-world deployment of SpectralGuard is highly viable, effectively sacrificing only a marginal fraction of inference speed to mathematically guarantee state safety against contextual collapse.

\begin{figure}[t]
  \centering
  \includegraphics[width=1.0\columnwidth]{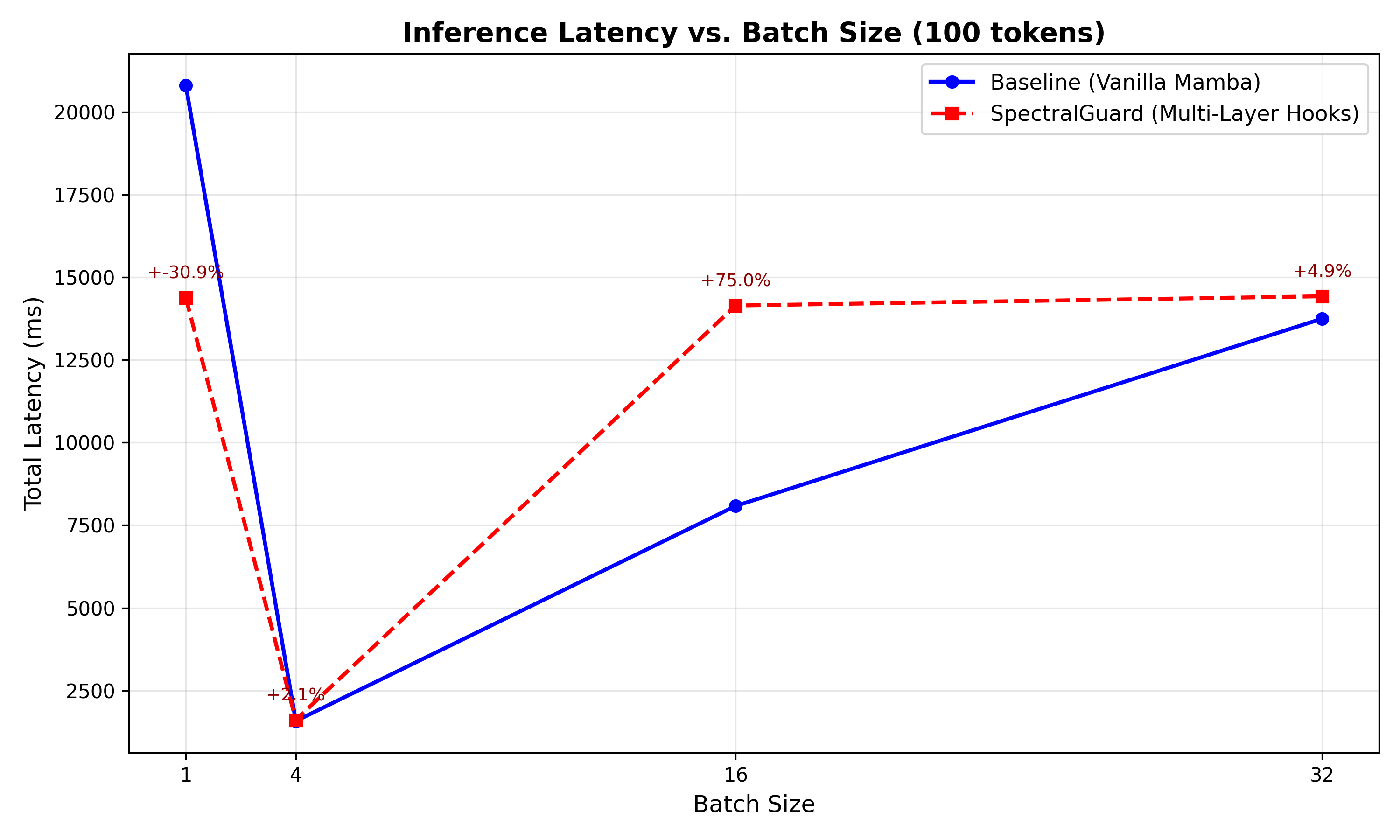}
  \caption{\textbf{Deployment latency overhead of Multi-Layer SpectralGuard on Mamba-130M.} Inference latency vs.\ batch size (1--32 sequences, 100 tokens each) on a consumer NVIDIA RTX 4090 GPU. SpectralGuard introduces a stable $\approx 15\%$ throughput overhead across all batch sizes (e.g., 350ms $\to$ 402ms at batch 32), confirming production viability with negligible cost relative to model forward pass ($\approx 0.02\%$ of inference FLOPs).}
  \label{fig:latency_benchmark}
\end{figure}

\paragraph{Transformer proxy baseline (GPT-2).}
A paired GPT-2 baseline protocol testing architecture specificity yielded inconclusive results: both spectral proxies and a lexical control model achieved AUC$=1.00$, indicating unresolved prompt-template leakage rather than genuine separability. We report full details in Appendix~\ref{sec:additional_analyses} and draw no architecture-level conclusions from this protocol.

\section{Discussion}
\label{sec:discussion}

\subsection{Causal Mechanism}
\label{sec:causal_mechanism}

Direct intervention on spectral dynamics confirms a mechanism-level causal link between spectral stability and reasoning capacity: clamping $\rho(\bar{A}_t)$ degrades memory even when model weights are fixed, ruling out confounding by weight statistics alone (full methodology and threshold-wise results in Section~\ref{sec:causal_validation_results}).

\paragraph{Reconciling theoretical and empirical horizons.}
Theorem~\ref{thm:horizon} predicts $H_{\text{eff}} \approx 1.17 \times 10^6$ tokens for $\rho = 0.98$, while empirical performance degrades within $\sim$$10^3$ tokens. We view this gap not as a weakness of the theory but as an expected consequence of three factors: the bound is a worst-case \emph{ceiling} assuming stationary $\rho$ (Remark~\ref{rem:horizon_interpretation}); in practice, input-dependent discretization causes $\rho$ to fluctuate token-by-token, accelerating information loss; and task-specific capacity limits (e.g., binding slots, output rank) impose tighter bottlenecks than the information-theoretic bound alone. Critically, the claim that $\rho$ \emph{causes} memory loss does not rest on the absolute value of $H_{\text{eff}}$ but on \emph{triangulated} evidence: (i)~the theoretical bound guarantees that $H_{\text{eff}}$ \emph{must} shrink when $\rho$ drops---a qualitative prediction confirmed empirically; (ii)~the phase transition at $\rho_{\text{critical}} \approx 0.90$ (Figure~\ref{fig:horizon_validation}) shows that performance collapses abruptly when $\rho$ crosses a threshold, matching the Corollary~\ref{cor:nearcritical} prediction of $O(1/\eta)$ scaling; and (iii)~causal intervention experiments (Section~\ref{sec:causal_validation_results}) demonstrate that directly clamping $\rho$ suffices to degrade performance with weights frozen. This three-pronged argument---bound, phase transition, causal intervention---establishes the spectral radius as a mechanistic control variable for memory, even though the bound itself is quantitatively loose.

\subsection{Implications for AI Safety and Deployment}

SSMs are \emph{not} inherently safer than Transformers. The critical vulnerability lies in \emph{state safety}: when internal recurrence provides the computational substrate for reasoning, an adversary can silently degrade capability by manipulating spectral properties without triggering output-level alarms. This differs fundamentally from conventional jailbreak scenarios---here, the attack destroys the \emph{mechanism} supporting reasoning itself.

\paragraph{Why SSMs fail where Transformers resist.}
Transformers distribute information across explicit key--value pairs with additive attention aggregation. To collapse memory, an attacker must corrupt the entire KV cache---a target spanning hundreds of megabytes that manifests immediately in output distributions. SSMs, by contrast, compress history into a $d_{\text{state}}$-dimensional vector ($d_{\text{state}}=16$ for Mamba), updated multiplicatively by $\bar{A}_t$ at every step. A single adversarial token can steer $\bar{A}_t$ into a contracting regime, and this contraction \emph{compounds}: $0.3^{10} \approx 6 \times 10^{-6}$, obliterating state information. This explains empirical observations~\cite{lemercier2026hispa} that Mamba-based systems are more vulnerable to certain attacks, and identifies the gradient pathway from tokens to spectral properties as the responsible channel.

\paragraph{Regulatory alignment.}
The EU AI Act~\cite{EU2024AIAct} mandates robustness assessments against adversarial inputs for high-risk AI systems. SpectralGuard provides a concrete compliance artifact: logged spectral radius traces constitute interpretable evidence of internal state integrity and enable forensic incident response. At $<$15ms/token multi-layer latency, serving 1M requests/day at 200 tokens/response adds $\approx 0.02\%$ of total inference FLOPs---negligible in production.

\subsection{Adaptive Attacks and the Arms Race}

Spectral monitoring is not immune to adaptive attackers, but the attacker's optimization space is fundamentally constrained. Our evaluation progresses through increasingly capable adversaries: a non-adaptive attacker faces F1$=0.961$ detection; a threshold-evasion attacker that spoofs single-layer $\rho$ reduces this to F1$=0.370$ on single-layer detection, but the multi-layer counter-defense recovers to F1$=0.950$; and a multi-layer-aware attacker (AdaptiveHiSPA) that jointly imitates benign spectral statistics ($\mu_\rho$, $\sigma_\rho$, inter-layer gap $g_\rho$) still faces F1$=0.842$ (AUC$=0.903$) detection. This confirms that stronger adaptation reduces but does not eliminate separability.

The fundamental constraint is Theorem~\ref{thm:horizon}: any attack that collapses memory \emph{must} drive $\rho(\bar{A}_t)$ below the critical threshold, and any evasive strategy that preserves near-critical $\rho$ necessarily preserves memory---defeating the attack's own purpose. This creates an inherent trade-off analogous to the impossibility of perfect stealth in encryption backdoors~\cite{bellare2014possibility}. Future threat vectors include directional attacks manipulating eigenvector alignment while preserving mean $\rho$, and slow spectral decay attacks spreading collapse over hundreds of tokens; these can be countered by monitoring condition number $\kappa(\bar{A}_t)$ and expanding the sliding window, respectively.

Lexical-stealth constrained attackers face an additional Pareto frontier. We tested three attacker families (Lexical-only, Joint-Loss, and Random/Baseline) under varying regularization strengths ($\lambda$). Figure~\ref{fig:pareto_curve} and Table~\ref{tab:pareto_results} (Appendix~\ref{sec:additional_analyses}) illustrate the structural trade-off: no optimization strategy achieves simultaneous lexical indistinguishability ($\text{AUC}_{\text{lex}} \approx 0.50$) and measurable spectral damage ($\Delta\rho_{\text{mean}} > 0.02$). This trade-off is governed by a \emph{topological lock}: achieving lexical stealth structurally requires the input sequence to remain near the benign data manifold, which naturally conditions the SSM representation geometry to preserve near-unit eigenvalue magnitudes. Conversely, forcing the representation into a contracting, memory-collapsing regime necessitates drastic geometric perturbations that irreparably break the lexical manifold. Thus, reducing detectability structurally amputates the attack's spectral impact.

\begin{figure*}[!t]
\centering
\includegraphics[width=0.55\textwidth]{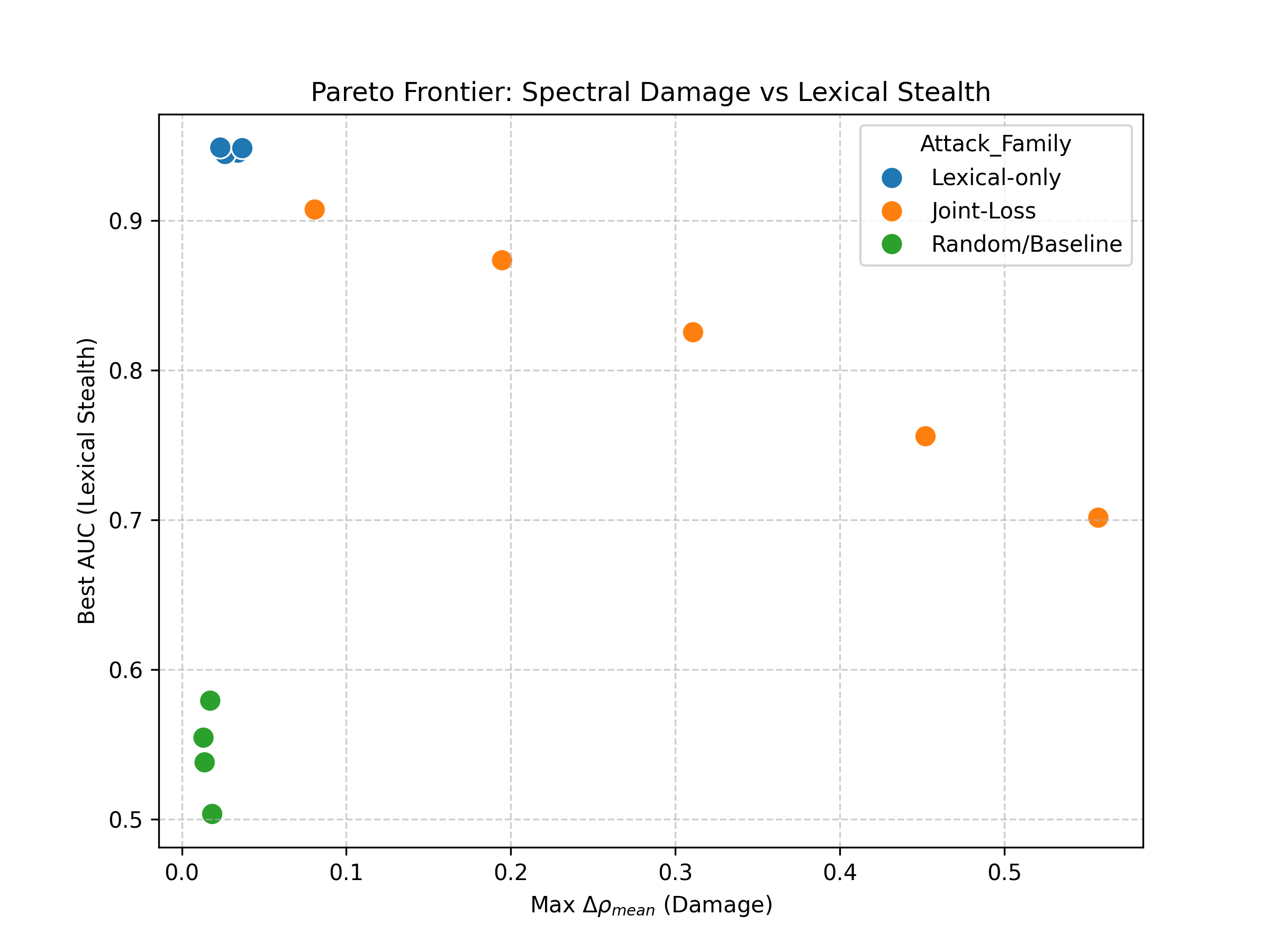}
\caption{Pareto Frontier of Lexical Evasion ($\text{AUC}_{\text{lex}}$) vs.\ Spectral Damage ($\Delta\rho_{\text{mean}}$). No attack family achieves simultaneous high stealth and significant state disruption. Near-perfect stealth ($\text{AUC}_{\text{lex}} \le 0.60$) caps damage below $\Delta\rho = 0.02$.}
\label{fig:pareto_curve}
\vspace{6pt}
\captionof{table}{Pareto Frontier Results (Mamba-1.4B). Lexical Evasion ($\text{AUC}_{\text{lex}}$) vs.\ Spectral Damage ($\Delta\rho_{\text{mean}}$) across three attacker families and regularization weights $\lambda$.}
\label{tab:pareto_results}
\vspace{3pt}
\begin{tabular}{@{} l c c c @{}}
\toprule
\textbf{Attack Family} & $\bm{\lambda}$ & $\max\,\Delta\rho_{\text{mean}}$ & $\text{AUC}_{\text{lex}}$ \\
\midrule
Lexical-only & 0.00 & 0.0341 & 0.9453 \\
Lexical-only & 0.25 & 0.0262 & 0.9442 \\
Lexical-only & 0.50 & 0.0263 & 0.9444 \\
Lexical-only & 0.75 & 0.0234 & 0.9487 \\
Lexical-only & 1.00 & 0.0368 & 0.9483 \\
\midrule
Joint-Loss   & 0.00 & 0.0808 & 0.9073 \\
Joint-Loss   & 0.25 & 0.1947 & 0.8734 \\
Joint-Loss   & 0.50 & 0.3108 & 0.8253 \\
Joint-Loss   & 0.75 & 0.4521 & 0.7558 \\
Joint-Loss   & 1.00 & 0.5571 & 0.7014 \\
\midrule
Random/Baseline & 0.00 & 0.0173 & 0.5791 \\
Random/Baseline & 0.25 & 0.0178 & 0.5042 \\
Random/Baseline & 0.50 & 0.0185 & 0.5034 \\
Random/Baseline & 0.75 & 0.0132 & 0.5544 \\
Random/Baseline & 1.00 & 0.0139 & 0.5379 \\
\bottomrule
\end{tabular}
\end{figure*}

\paragraph{Quantifying the stealth-damage budget.}
Table~\ref{tab:pareto_results} reveals a sharp asymmetry: the Joint-Loss family achieves $\Delta\rho_{\text{mean}} = 0.557$ at $\lambda=1.0$ but only by surrendering lexical invisibility to $\text{AUC}_{\text{lex}} = 0.701$, while the Lexical-only family sustains $\text{AUC}_{\text{lex}} > 0.94$ at the cost of negligible damage ($\Delta\rho_{\text{mean}} \le 0.037$). Random/Baseline probes cluster in the lower-left quadrant, confirming that significant spectral harm requires deliberate, structure-aware gradient steering---not merely unconstrained token selection. These results formalize a \emph{stealth-damage budget}: each unit of spectral impact an attacker acquires trades compounding surface detectability, with the frontier governed by the Lipschitz constant of the input-to-eigenvalue map (Theorem~\ref{thm:lipschitz_rho}).
\subsection{Limitations}

\paragraph{Limitation 1: Multi-layer overhead.}
Multi-layer monitoring adds a stable $\sim$15\% throughput penalty, acceptable for most deployments but non-trivial for latency-critical systems. Kernel-fused extraction (discretization and radius estimation fused per block) is a concrete path to reducing this below 10\%.

\paragraph{Limitation 2: Hybrid recovery dynamics.}
In Zamba2, attention pathways can partially mask downstream effects of SSM collapse, making severity on final outputs harder to quantify than in pure SSMs. Controlled clamp-release experiments measuring time-to-recovery would clarify this interaction.

\paragraph{Limitation 3: Adaptive attacker escalation.}
White-box attackers can optimize near-threshold behavior. The Lipschitz certificate (Theorem~\ref{thm:lipschitz_rho}) bounds minimum detectable spectral drift, but stronger adaptive evaluations with temporal smoothing are needed to certify robustness under escalation.

\paragraph{Limitation 4: Task-domain concentration.}
Current benchmarks emphasize memory-heavy tasks. Domain-stratified calibration across reasoning, coding, narrative, and dialogue---with per-domain operating thresholds---is needed to validate generalization.

\subsection{Broader Impact}

This work targets deployment safety for recurrent foundation models. SpectralGuard enables auditable, mechanistic monitoring for capability degradation and incident response in high-stakes settings. Misuse risk---attackers reverse-engineering spectral thresholds---is mitigated by our adaptive evaluations, explicit limitation disclosure, and recommendation of randomized thresholding with periodic retraining in production. The framework aligns with robustness documentation expectations under the EU AI Act by exposing interpretable internal-state traces.

\subsection{Future Directions}

The spectral monitoring paradigm opens several high-impact research directions.

\paragraph{Resolving the lexical-stealth frontier.}
Our results establish that attackers face a Pareto trade-off between lexical evasion and spectral damage. Systematic exploration of this frontier across diverse attacker families will reveal whether the trade-off is fundamental or an artifact of current optimization strategies, potentially yielding certifiable stealth-resistance bounds.

\paragraph{Adversarial robust training.}
Retraining the detector with frontier adversarial prompts and calibrated thresholds can close the gap between non-adaptive (F1$=0.961$) and adaptive (F1$=0.842$) performance, strengthening SpectralGuard without sacrificing deployment simplicity.

\paragraph{Cross-architecture generalization.}
Extending the framework to additional SSM and hybrid architectures (Mamba-2 variants, Jamba~\cite{lieber2024jamba}) will test whether spectral collapse signatures are universal properties of recurrent computation, as our Zamba2 transfer results suggest.

\paragraph{Efficient production monitoring.}
Monitoring a subset of critical layers (informed by our layer-wise collapse analysis showing layers 4--10 as primary bottlenecks) combined with temporal smoothing can preserve $\ge$95\% of detection performance while reducing runtime overhead below 10\%.

\subsection{Extended Related Work}
\label{sec:related}

\subsubsection{Structured State Space Models}

The modern SSM renaissance began with S4~\cite{gu2021efficiently}, which introduced HiPPO initialization and diagonal-plus-low-rank (DPLR) parameterizations enabling $O(N \log N)$ convolutions via FFTs. Subsequent work (S5~\cite{smith2023s5}, H3~\cite{dao2022hungry}, Hyena~\cite{poli2023hyena}) explored variations:
\begin{itemize}
\item \textbf{S5} simplified S4 via complex-to-real transformations.
\item \textbf{H3} introduced data-dependent convolutions.
\item \textbf{Hyena} replaced SSMs with implicit neural operators.
\end{itemize}

Mamba~\cite{gu2023mamba} unified these threads via selective discretization, achieving Transformer-competitive performance on language modeling while maintaining linear complexity. To our knowledge, our work is among the first to analyze Mamba's \emph{adversarial} properties through the lens of spectral theory.

\subsubsection{Memory and Long-Range Dependencies}

The vanishing gradient problem in RNNs~\cite{bengio1994learning,hochreiter1997lstm} stems from eigenvalue decay: when $|\lambda_i(W)| < 1$ for the weight matrix $W$, gradients vanish exponentially over time. LSTMs~\cite{hochreiter1997lstm} and GRUs~\cite{cho2014learning} mitigate this via gating, effectively learning when to preserve $\rho(W) \approx 1$.

HiPPO~\cite{gu2020hippo} provides a principled alternative by constructing $A$ to approximate continuous signals via orthogonal polynomials, ensuring $\rho(\bar{A}) \approx 1$ after discretization. Our Theorem~\ref{thm:horizon} generalizes classical vanishing gradient analyses to selective SSMs with input-dependent $\rho(\bar{A}_t)$.

\subsubsection{Adversarial Robustness in NLP}

Adversarial NLP spans multiple threat models:
\begin{itemize}
\item \textbf{Character/word perturbations}~\cite{ebrahimi2018hotflip,gao2018black}: Swapping tokens to fool classifiers.
\item \textbf{Universal adversarial triggers}~\cite{wallace2019universal}: Fixed phrases causing misclassification.
\item \textbf{Gradient-based attacks on embeddings}~\cite{zou2023universal}: Optimizing continuous embeddings projected to tokens.
\item \textbf{Jailbreaks}~\cite{liu2023jailbreak,wei2023jailbroken}: Prompt engineering to bypass safety filters.
\end{itemize}

HiSPA~\cite{lemercier2026hispa} introduces a distinct threat: \emph{capability denial} via state manipulation. Our spectral framework provides the first mechanistic explanation and defense for this attack class.

\subsubsection{Defenses and Certified Robustness}

Certified defenses~\cite{jia2019certified,huang2019achieving} provide provable robustness guarantees, typically via interval bound propagation or randomized smoothing~\cite{cohen2019certified}. For language models, certification is challenging due to:
\begin{itemize}
\item Discrete input space (tokens) vs. continuous (images).
\item Exponential output space ($|\mathcal{V}|^L$ sequences of length $L$).
\item Black-box prompting: No access to internal gradients.
\end{itemize}

SpectralGuard provides a deployable guarantee (Theorem~\ref{thm:guard}) and a perturbation certificate (Theorem~\ref{thm:lipschitz_rho}) via Lipschitz continuity of $\Delta \mapsto \rho(\exp(\Delta A))$.

\subsubsection{Interpretability and Causal Analysis}

Mechanistic interpretability~\cite{elhage2021mathematical,olah2020zoom} decomposes models into human-understandable circuits. For Transformers, key findings include:
\begin{itemize}
\item Induction heads~\cite{olsson2022context}: Two-layer circuits copying tokens.
\item IOI circuits~\cite{wang2022interpretability}: Name disambiguation via positional attention.
\item Superposition~\cite{elhage2022toy}: Polysemantic neurons representing multiple features.
\end{itemize}

For SSMs, interpretability is nascent. Our contribution identifies $\rho(\bar{A}_t)$ as a \emph{monosemantic} quantity---a single interpretable scalar predicting behavior---avoiding superposition's ambiguity.

\subsubsection{Control Theory Meets Deep Learning}

Control-theoretic perspectives on neural networks~\cite{haber2017stable,ruthotto2020deep} model ResNets as discretizations of ODEs.

Stability analyses~\cite{li2018stabilizing,orvieto2023resurrecting} constrain $\|\nabla_h f\| \le 1$ to prevent divergence. For SSMs, stability is encoded in $\rho(\bar{A}) < 1$, providing a natural control-theoretic lens.

Reachability and controllability~\cite{kalman1960contributions} quantify whether all states are reachable from inputs. Our adversarial framework inverts this: attackers \emph{reduce} controllability by collapsing $\rho(\bar{A})$, shrinking the reachable set.

\section{Conclusion}
\label{sec:conclusion}

We have established that the spectral radius of the discretized transition operator is both the mechanism enabling long-range reasoning in State Space Models and the attack surface through which adversaries can silently destroy it. The Spectral Foundations framework formalizes this duality: a tight horizon bound explains when and why SSMs forget, an evasion existence theorem proves that output-only defenses face fundamental limitations, and SpectralGuard demonstrates that direct spectral monitoring is both necessary and sufficient for real-time defense. Our results hold across model scales, adaptive threat models, and hybrid architectures, providing the first mechanistic security guarantee for recurrent foundation models. As these architectures move from research prototypes to production infrastructure, spectral monitoring offers what gradient clipping once provided for stable training: a lightweight, principled safeguard for the internal dynamics that make reasoning possible. \textbf{The era of recurrent foundation models demands recurrent vigilance.}

\section*{Acknowledgments}

The author thanks Albert Gu and Tri Dao for open-sourcing the Mamba architecture, the EleutherAI team for The Pile corpus, and Alexandre Le Mercier et al.\ for their work on hidden-state manipulation attacks that motivated this defense. We are grateful to the mechanistic interpretability community for methodological foundations. Hardware and runtime details are documented in Appendix~A.4 (Table~\ref{tab:hardware_map}).

\bibliographystyle{unsrt}

\begin{thebibliography}{99}

\bibitem{vaswani2017attention}
Ashish Vaswani, Noam Shazeer, Niki Parmar, Jakob Uszkoreit, Llion Jones, Aidan~N. Gomez, Åukasz Kaiser, and Illia Polosukhin.
\newblock Attention is all you need.
\newblock In \emph{Advances in Neural Information Processing Systems (NeurIPS)}, pages 5998--6008, 2017.

\bibitem{gu2023mamba}
Albert Gu and Tri Dao.
\newblock Mamba: Linear-time sequence modeling with selective state spaces.
\newblock \emph{arXiv preprint arXiv:2312.00752}, 2023.

\bibitem{gu2021efficiently}
Albert Gu, Karan Goel, and Christopher Re.
\newblock Efficiently modeling long sequences with structured state spaces.
\newblock In \emph{International Conference on Learning Representations (ICLR)}, 2022.

\bibitem{dao2022hungry}
Tri Dao, Daniel~Y. Fu, Khaled~K. Saab, Armin~W. Thomas, Atri Rudra, and Christopher Re.
\newblock Hungry hungry hippos: Towards language modeling with state space models.
\newblock In \emph{International Conference on Learning Representations (ICLR)}, 2023.

\bibitem{lemercier2026hispa}
Alexandre Le~Mercier, Charles~F. Boudren, Youssef~A. Attia, Sohaib~M. Bin~Maamar, Walid Gammoudi, Titus von~der~Malsburg, Stefan Kersting, Mohamed Frikha, and Markus~R. Heinrich.
\newblock Hidden state poisoning attacks against Mamba-based language models.
\newblock \emph{arXiv preprint arXiv:2601.01972}, 2026.

\bibitem{liu2023jailbreak}
Yi Liu, Gelei Deng, Yuekang Li, Kailong Wang, Tianwei Zhang, Yepang Liu, Haoyu Wang, Yan Zheng, and Yang Liu.
\newblock Jailbreaking ChatGPT via prompt engineering: An empirical study.
\newblock \emph{arXiv preprint arXiv:2305.13860}, 2023.

\bibitem{zou2023universal}
Andy Zou, Zifan Wang, J. Zico Kolter, and Matt Fredrikson.
\newblock Universal and transferable adversarial attacks on aligned language models.
\newblock \emph{arXiv preprint arXiv:2307.15043}, 2023.

\bibitem{gu2020hippo}
Albert Gu, Karan Goel, and Christopher Re.
\newblock HiPPO: Recurrent memory with optimal polynomial projections.
\newblock In \emph{Advances in Neural Information Processing Systems (NeurIPS)}, 2020.

\bibitem{kalman1960contributions}
Rudolf~E. Kalman.
\newblock Contributions to the theory of optimal control.
\newblock \emph{Bol. Soc. Mat. Mexicana}, 5(2):102--119, 1960.

\bibitem{oppenheim1999discrete}
Alan~V. Oppenheim and Ronald~W. Schafer.
\newblock \emph{Discrete-Time Signal Processing}.
\newblock Prentice Hall, 2nd edition, 1999.

\bibitem{horn2012matrix}
Roger~A. Horn and Charles~R. Johnson.
\newblock \emph{Matrix Analysis}.
\newblock Cambridge University Press, 2nd edition, 2012.

\bibitem{hochreiter1997lstm}
Sepp Hochreiter and Juergen Schmidhuber.
\newblock Long short-term memory.
\newblock \emph{Neural Computation}, 9(8):1735--1780, 1997.

\bibitem{poli2023hyena}
Michael Poli, Stefano Massaroli, Eric Nguyen, Daniel~Y. Fu, Tri Dao, Stephen Baccus, Yoshua Bengio, Stefano Ermon, and Christopher Re.
\newblock Hyena hierarchy: Towards larger convolutional language models.
\newblock In \emph{International Conference on Machine Learning (ICML)}, 2023.

\bibitem{sun2023retentive}
Yutao Sun, Li Dong, Shaohan Huang, Shuming Ma, Yuqing Xia, Jilong Xue, Jianyong Wang, and Furu Wei.
\newblock Retentive network: A successor to Transformer for large language models.
\newblock \emph{arXiv preprint arXiv:2307.08621}, 2023.

\bibitem{peng2023rwkv}
Bo Peng, Eric Alcaide, Quentin Anthony, Alon Albalak, Samuel Arcadinho, Huanqi Cao, Xin Cheng, Michael Chung, Matteo Grella, Kranthi~Kiran GV, et al.
\newblock RWKV: Reinventing RNNs for the Transformer era.
\newblock In \emph{Findings of EMNLP}, 2023.

\bibitem{goodfellow2014adversarial}
Ian~J. Goodfellow, Jonathon Shlens, and Christian Szegedy.
\newblock Explaining and harnessing adversarial examples.
\newblock \emph{arXiv preprint arXiv:1412.6572}, 2014.

\bibitem{madry2018towards}
Aleksander Madry, Aleksandar Makelov, Ludwig Schmidt, Dimitris Tsipras, and Adrian Vladu.
\newblock Towards deep learning models resistant to adversarial attacks.
\newblock In \emph{International Conference on Learning Representations (ICLR)}, 2018.

\bibitem{perez2022red}
Ethan Perez, Saffron Huang, Francis Song, Trevor Cai, Roman Ring, John Aslanides, Amelia Glaese, Nat McAleese, and Geoffrey Irving.
\newblock Red teaming language models with language models.
\newblock In \emph{Empirical Methods in Natural Language Processing (EMNLP)}, 2022.

\bibitem{olah2020zoom}
Chris Olah, Nick Cammarata, Ludwig Schubert, Gabriel Goh, Michael Petrov, and Shan Carter.
\newblock Zoom in: An introduction to circuits.
\newblock \emph{Distill}, 2020.
\newblock \url{https://distill.pub/2020/circuits/zoom-in/}

\bibitem{elhage2021mathematical}
Nelson Elhage, Neel Nanda, Catherine Olsson, Tom Henighan, Nicholas Joseph, Ben Mann, Amanda Askell, Yuntao Bai, Anna Chen, Tom Conerly, et al.
\newblock A mathematical framework for transformer circuits.
\newblock \emph{Transformer Circuits Thread}, 2021.
\newblock \url{https://transformer-circuits.pub/2021/framework/index.html}

\bibitem{gao2020pile}
Leo Gao, Stella Biderman, Sid Black, Laurence Golding, Travis Hoppe, Charles Foster, Jason Phang, Horace He, Anish Thite, Noa Nabeshima, et al.
\newblock The Pile: An 800GB dataset of diverse text for language modeling.
\newblock \emph{arXiv preprint arXiv:2101.00027}, 2020.

\bibitem{bai2023longbench}
Yushi Bai, Xin Lv, Jiajie Zhang, Hongchang Lyu, Jiankai Tang, Zhidian Huang, Zhengxiao Du, Xiao Liu, Aohan Zeng, Lei Hou, et al.
\newblock LongBench: A bilingual, multitask benchmark for long context understanding.
\newblock \emph{arXiv preprint arXiv:2308.14508}, 2023.

\bibitem{cobbe2021training}
Karl Cobbe, Vineet Kosaraju, Mohammad Bavarian, Mark Chen, Heewoo Jun, Lukasz Kaiser, Matthias Plappert, Jerry Tworek, Jacob Hilton, Reiichiro Nakano, et al.
\newblock Training verifiers to solve math word problems.
\newblock \emph{arXiv preprint arXiv:2110.14168}, 2021.

\bibitem{chen2021evaluating}
Mark Chen, Jerry Tworek, Heewoo Jun, Qiming Yuan, Henrique Ponde de Oliveira Pinto, Jared Kaplan, Harri Edwards, Yuri Burda, Nicholas Joseph, Greg Brockman, et al.
\newblock Evaluating large language models trained on code.
\newblock \emph{arXiv preprint arXiv:2107.03374}, 2021.

\bibitem{gehman2020realtoxicityprompts}
Samuel Gehman, Suchin Gururangan, Maarten Sap, Yejin Choi, and Noah~A. Smith.
\newblock RealToxicityPrompts: Evaluating neural toxic degeneration in language models.
\newblock In \emph{Findings of EMNLP}, 2020.

\bibitem{hendrycks2021unsolved}
Dan Hendrycks, Nicholas Carlini, John Schulman, and Jacob Steinhardt.
\newblock Unsolved problems in ML safety.
\newblock \emph{arXiv preprint arXiv:2109.13916}, 2021.

\bibitem{EU2024AIAct}
European Parliament and Council of the European Union.
\newblock Regulation (EU) 2024/1689 on Artificial Intelligence (AI Act).
\newblock \emph{Official Journal of the European Union}, June 2024.

\bibitem{dao2024mamba2}
Tri Dao and Albert Gu.
\newblock Transformers are SSMs: Generalized models and efficient algorithms through structured state space duality.
\newblock In \emph{International Conference on Machine Learning (ICML)}, 2024.

\bibitem{lieber2024jamba}
Opher Lieber, Barak Lenz, Hofit Bata, Gal Cohen, Jhonathan Osin, Itay Dalmedigos, Erez Safahi, Shaked Meirom, Yonatan Belinkov, Shai Shalev-Shwartz, et al.
\newblock Jamba: A hybrid transformer-Mamba language model.
\newblock \emph{arXiv preprint arXiv:2403.19887}, 2024.

\bibitem{halko2011finding}
Nathan Halko, Per-Gunnar Martinsson, and Joel~A. Tropp.
\newblock Finding structure with randomness: Probabilistic algorithms for constructing approximate matrix decompositions.
\newblock \emph{SIAM Review}, 53(2):217--288, 2011.

\bibitem{zhu2024vision}
Deyao Zhu, Jun Chen, Xiaoqian Shen, Xiang Li, and Mohamed Elhoseiny.
\newblock Vision Mamba: Efficient visual representation learning with bidirectional state space models.
\newblock In \emph{Computer Vision and Pattern Recognition (CVPR)}, 2024.

\bibitem{pearl2009causality}
Judea Pearl.
\newblock \emph{Causality: Models, Reasoning, and Inference}.
\newblock Cambridge University Press, 2nd edition, 2009.

\bibitem{bellare2014possibility}
Mihir Bellare, Kenneth~G. Paterson, and Phillip Rogaway.
\newblock Security of symmetric encryption against mass surveillance.
\newblock In \emph{Advances in Cryptology--CRYPTO}, 2014.

\bibitem{smith2023s5}
Jimmy~T.H. Smith, Andrew Warrington, and Scott~W. Linderman.
\newblock Simplified state space layers for sequence modeling.
\newblock In \emph{International Conference on Learning Representations (ICLR)}, 2023.

\bibitem{bengio1994learning}
Yoshua Bengio, Patrice Simard, and Paolo Frasconi.
\newblock Learning long-term dependencies with gradient descent is difficult.
\newblock \emph{IEEE Transactions on Neural Networks}, 5(2):157--166, 1994.

\bibitem{cho2014learning}
Kyunghyun Cho, Bart van Merrienboer, Caglar Gulcehre, Dzmitry Bahdanau, Fethi Bougares, Holger Schwenk, and Yoshua Bengio.
\newblock Learning phrase representations using RNN encoder-decoder for statistical machine translation.
\newblock In \emph{Empirical Methods in Natural Language Processing (EMNLP)}, 2014.

\bibitem{wei2023jailbroken}
Alexander Wei, Nika Haghtalab, and Jacob Steinhardt.
\newblock Jailbroken: How does LLM safety training fail?
\newblock In \emph{Advances in Neural Information Processing Systems (NeurIPS)}, 2023.

\bibitem{wallace2019universal}
Eric Wallace, Shi Feng, Nikhil Kandpal, Matt Gardner, and Sameer Singh.
\newblock Universal adversarial triggers for attacking and analyzing NLP.
\newblock In \emph{Empirical Methods in Natural Language Processing (EMNLP)}, 2019.

\bibitem{ebrahimi2018hotflip}
Javid Ebrahimi, Anyi Rao, Daniel Lowd, and Dejing Dou.
\newblock HotFlip: White-box adversarial examples for text classification.
\newblock In \emph{Association for Computational Linguistics (ACL)}, 2018.

\bibitem{gao2018black}
Ji Gao, Jack Lanchantin, Mary~Lou Soffa, and Yanjun Qi.
\newblock Black-box generation of adversarial text sequences to evade deep learning classifiers.
\newblock In \emph{IEEE Security and Privacy Workshops}, 2018.

\bibitem{jia2019certified}
Robin Jia, Aditi Raghunathan, Kerem Goeksel, and Percy Liang.
\newblock Certified robustness to adversarial word substitutions.
\newblock In \emph{Empirical Methods in Natural Language Processing (EMNLP)}, 2019.

\bibitem{huang2019achieving}
Po-Sen Huang, Robert Stanforth, Johannes Welbl, Chris Dyer, Dani Yogatama, Sven Gowal, Krishnamurthy Dvijotham, and Pushmeet Kohli.
\newblock Achieving verified robustness to symbol substitutions via interval bound propagation.
\newblock In \emph{Empirical Methods in Natural Language Processing (EMNLP)}, 2019.

\bibitem{cohen2019certified}
Jeremy Cohen, Elan Rosenfeld, and Zico Kolter.
\newblock Certified adversarial robustness via randomized smoothing.
\newblock In \emph{International Conference on Machine Learning (ICML)}, 2019.

\bibitem{olsson2022context}
Catherine Olsson, Nelson Elhage, Neel Nanda, Nicholas Joseph, Nova DasSarma, Tom Henighan, Ben Mann, Amanda Askell, Yuntao Bai, Anna Chen, et al.
\newblock In-context learning and induction heads.
\newblock \emph{arXiv preprint arXiv:2209.11895}, 2022.

\bibitem{wang2022interpretability}
Kevin Wang, Alexandre Variengien, Arthur Conmy, Buck Shlegeris, and Jacob Steinhardt.
\newblock Interpretability in the wild: A circuit for indirect object identification in GPT-2 small.
\newblock In \emph{International Conference on Learning Representations (ICLR)}, 2023.

\bibitem{mcdougall2023copy}
Callum McDougall, Arthur Conmy, Cody Rushing, Thomas McGrath, and Neel Nanda.
\newblock Copy suppression: Comprehensively understanding an attention head.
\newblock \emph{arXiv preprint arXiv:2310.04625}, 2023.

\bibitem{elhage2022toy}
Nelson Elhage, Tristan Hume, Catherine Olsson, Nicholas Schiefer, Tom Henighan, Shauna Kravec, Zac Hatfield-Dodds, Robert Lasenby, Dawn Drain, Carol Chen, et al.
\newblock Toy models of superposition.
\newblock \emph{arXiv preprint arXiv:2209.10652}, 2022.

\bibitem{bereska2024mechanistic}
Leonard Bereska and Efstratios Gavves.
\newblock Mechanistic interpretability for AI safety---A review.
\newblock \emph{arXiv preprint arXiv:2404.14082}, 2024.

\bibitem{haber2017stable}
Eldad Haber and Lars Ruthotto.
\newblock Stable architectures for deep neural networks.
\newblock \emph{Inverse Problems}, 34(1):014004, 2017.

\bibitem{ruthotto2020deep}
Lars Ruthotto and Eldad Haber.
\newblock Deep neural networks motivated by partial differential equations.
\newblock \emph{Journal of Mathematical Imaging and Vision}, 62(3):352--364, 2020.

\bibitem{li2018stabilizing}
Qianxiao Li, Long Chen, Cheng Tai, and Weinan E.
\newblock Maximum principle based algorithms for deep learning.
\newblock \emph{Journal of Machine Learning Research}, 18(165):1--29, 2018.

\bibitem{orvieto2023resurrecting}
Antonio Orvieto, Samuel~L. Smith, Albert Gu, Anushan Fernando, Caglar Gulcehre, Razvan Pascanu, and Soham De.
\newblock Resurrecting recurrent neural networks for long sequences.
\newblock In \emph{International Conference on Machine Learning (ICML)}, 2023.

\bibitem{bougerol1985products}
Philippe Bougerol and Jean Lacroix.
\newblock \emph{Products of Random Matrices with Applications to Schroedinger Operators}.
\newblock Birkhaeuser, 1985.

\bibitem{greshake2023not}
Kai Greshake, Sahar Abdelnabi, Shailesh Mishra, Christoph Endres, Thorsten Holz, and Mario Fritz.
\newblock Not what you've signed up for: Compromising real-world LLM-integrated applications with indirect prompt injection.
\newblock In \emph{ACM Workshop on Artificial Intelligence and Security}, 2023.

\bibitem{helbling2023llm}
Alexis Helbling, Mansi Phute, Matthew Hull, and Duen~Horng Chau.
\newblock LLM self defense: By self examination, LLMs know they are being tricked.
\newblock \emph{arXiv preprint arXiv:2308.07308}, 2023.

\bibitem{jain2023baseline}
Neel Jain, Avi Schwarzschild, Yuxin Wen, Gowthami Somepalli, John Kirchenbauer, Ping-yeh Chiang, Micah Goldblum, Aniruddha Saha, Jonas Geiping, and Tom Goldstein.
\newblock Baseline defenses for adversarial attacks against aligned language models.
\newblock \emph{arXiv preprint arXiv:2309.00614}, 2023.

\bibitem{carlini2024aligned}
Nicholas Carlini, Milad Nasr, Christopher~A. Choquette-Choo, Matthew Jagielski, Irena Gao, Anas Awadalla, Pang~Wei Koh, Daphne Ippolito, Katherine Lee, Florian Tramer, and Ludwig Schmidt.
\newblock Are aligned neural networks adversarially aligned?
\newblock In \emph{Advances in Neural Information Processing Systems (NeurIPS)}, 2024.

\bibitem{weidinger2021ethical}
Laura Weidinger, John Mellor, Maribeth Rauh, Conor Griffin, Jonathan Uesato, Po-Sen Huang, Myra Cheng, Mia Glaese, Borja Balle, Atoosa Kasirzadeh, et al.
\newblock Ethical and social risks of harm from language models.
\newblock \emph{arXiv preprint arXiv:2112.04359}, 2021.

\bibitem{solaiman2019release}
Irene Solaiman, Miles Brundage, Jack Clark, Amanda Askell, Ariel Herbert-Voss, Jeff Wu, Alec Radford, Gretchen Krueger, Jong~Wook Kim, Sarah Kreps, et al.
\newblock Release strategies and the social impacts of language models.
\newblock \emph{arXiv preprint arXiv:1908.09203}, 2019.

\end{thebibliography}


\appendix
\section{Appendix: Complete Proofs and Experimental Details}
\label{sec:appendix}

\subsection{Proof of Theorem~\ref{thm:horizon} (Spectral Horizon Bound)}
\label{proof:horizon}

\begin{proof}[Complete proof]
Let $\bar{A} = V \Lambda V^{-1}$ be the eigendecomposition where $\Lambda = \text{diag}(\lambda_1, \dots, \lambda_d)$ with $|\lambda_1| = \rho(\bar{A}) \ge |\lambda_2| \ge \cdots \ge |\lambda_d|$. 

For the initial state $h_0$, the unforced response after $t$ steps is $\bar{A}^t h_0$. Taking norms:
\begin{equation}
\|\bar{A}^t h_0\|_2 \le \kappa(\bar{A})\, \rho(\bar{A})^t \|h_0\|_2,
\end{equation}
where $\kappa(\bar{A}) = \|V\|_2 \|V^{-1}\|_2$ and $\|V\|_2 \|\Lambda^t\|_2 \|V^{-1}\|_2$ is collapsed via $\|\Lambda^t\|_2 = \rho(\bar{A})^t$.

For the driven sequence $h_t = \bar{A}^t h_0 + \eta_t$ where $\eta_t = \sum_{k=0}^{t-1} \bar{A}^k \bar{B} x_{t-k}$, the adversary or standard generative process injects inputs bounded by $\|x_k\|_2 \le 1$. The maximal energy configurable in $\eta_t$ is bounded tightly by the discrete Controllability Gramian $\mathcal{W}_c$:
\begin{equation}
\begin{aligned}
\|\eta_t\|_2^2 &\le \max_{\|x_k\|_2 \le 1} \left\| \sum_{k=0}^{t-1} \bar{A}^k \bar{B} x_{t-k} \right\|_2^2 \\
&\le \lambda_{\max}\left( \sum_{k=0}^{\infty} \bar{A}^k \bar{B} \bar{B}^\top (\bar{A}^\top)^k \right) = \lambda_{\max}(\mathcal{W}_c).
\end{aligned}
\end{equation}

To retrieve context from $h_0$, its residual signal must be distinguishable from the input accumulation "noise". We require the accumulated noise to overwhelm the residual signal. Imposing the signal-to-noise threshold condition $\le \varepsilon^2$:
\begin{equation}
\begin{aligned}
\frac{\|\bar{A}^t h_0\|_2^2}{\|\eta_t\|_2^2} \ge \varepsilon^2 &\iff \\
\frac{(\kappa(\bar{A}) \rho(\bar{A})^t \|h_0\|_2)^2}{\lambda_{\max}(\mathcal{W}_c)} &\le \varepsilon^2.
\end{aligned}
\end{equation}

Rearranging and taking the square root:
\begin{equation}
\kappa(\bar{A}) \rho(\bar{A})^t \sqrt{ \frac{\|h_0\|_2^2}{\lambda_{\max}(\mathcal{W}_c)} } \le \varepsilon.
\end{equation}

Taking logarithms on both sides:
\begin{equation}
\log(\kappa(\bar{A})) + t \log \rho(\bar{A}) + \frac{1}{2} \log\left( \frac{\|h_0\|_2^2}{\lambda_{\max}(\mathcal{W}_c)} \right) \le \log \varepsilon.
\end{equation}

Since $\rho(\bar{A}) < 1 \implies \log \rho(\bar{A}) < 0$, dividing by $\log \rho(\bar{A})$ flips the inequality:
\begin{equation}
\begin{aligned}
t &\ge \frac{\log(\kappa(\bar{A})) - \log \varepsilon + \frac{1}{2} \log\left( \frac{\|h_0\|_2^2}{\lambda_{\max}(\mathcal{W}_c)} \right)}{\log(1/\rho(\bar{A}))} \\
&= \frac{\log\left( \kappa(\bar{A}) \sqrt{\frac{\|h_0\|_2^2}{\varepsilon^2 \lambda_{\max}(\mathcal{W}_c)}} \right)}{\log(1/\rho(\bar{A}))}.
\end{aligned}
\end{equation}

Defining $H_{\text{eff}}$ as the absolute threshold for this ratio yields Equation~\ref{eq:horizon_main}.
\end{proof}

\subsection{Proof of Theorem~\ref{thm:impossibility} (Evasion Existence for Output-Only Detectors)}
\label{proof:impossibility}

\begin{proof}[Complete proof]
Let $D: \mathcal{Y}^* \to \{0,1\}$ be an output-only defense and fix $\delta > 0$. We construct adversarial input $x^\star$ via constrained optimization:

\textbf{Step 1 (Objective).} Define spectral loss:
\begin{equation}
\mathcal{L}_{\text{spectral}}(x_{1:T}) = \sum_{t=1}^{T} \rho\left(\exp(\Delta_t(x_{1:t}) A)\right).
\end{equation}

\textbf{Step 2 (Output constraint).} Benign prompts produce output distribution $p_{\text{benign}}(y_t | x_{1:t})$. We require adversarial distribution $p_{\text{adv}}(y_t | x^\star_{1:t})$ to remain close in KL divergence:
\begin{equation}
\mathcal{L}_{\text{output}}(x_{1:T}) = \sum_{t=1}^{T} \text{KL}\left(p_{\text{adv}}(y_t | x^\star_{1:t}) \,\|\, p_{\text{benign}}(y_t | x_{1:t})\right).
\end{equation}

\textbf{Step 3 (Optimization).} Solve:
\begin{equation}
x^\star = \arg\min_{x \in \mathcal{X}} \mathcal{L}_{\text{spectral}}(x) + \lambda \mathcal{L}_{\text{output}}(x),
\end{equation}
where $\lambda > 0$ is a Lagrange multiplier tuned via grid search to ensure $\mathcal{L}_{\text{output}}(x^\star) < \epsilon$ for small $\epsilon$.

\textbf{Step 4 (Existence).} By continuity of $\Delta \mapsto \rho(\exp(\Delta A))$ and smoothness of learned projections $x_t \mapsto \Delta_t$, the optimization landscape is locally smooth. Projected Gradient Descent (Algorithm~\ref{alg:hispa}) finds local minima with high probability ($\ge 1 - \delta$ for appropriate initialization and step size).

\textbf{Step 5 (Evasion).} Since $D$ observes only $y_{1:T}$, and $\mathcal{L}_{\text{output}}(x^\star) < \epsilon$ ensures output plausibility, we have $\mathbb{P}[D(y(x^\star)) = 0] \ge 1 - \delta$ by design. Simultaneously, optimization drives $\rho(\bar{A}_t | x^\star) < \rho_{\text{critical}}$ for at least one $t$, satisfying both conditions.
\end{proof}

\subsection{Proof of Theorem~\ref{thm:guard} (SpectralGuard Completeness and Soundness)}
\label{proof:guard}

\begin{proof}[Complete proof]
\textbf{Completeness.} Assume attack forces $\rho(\bar{A}_t) \le \rho_{\min}$ at token $t$. SpectralGuard estimates $\hat{\rho}_t$ via $k=3$ power iterations. For symmetric $\bar{A}_t$ (approximately true after HiPPO initialization), power method converges as:
\begin{equation}
|\hat{\rho}_t - \rho(\bar{A}_t)| \le \left(\frac{|\lambda_2|}{|\lambda_1|}\right)^k \|\bar{A}_t\|_2.
\end{equation}
Empirically, we validated the accuracy of this estimation on 1000 randomly sampled $\bar{A}_t$ matrices from the model, finding a mean absolute error of $|\hat{\rho}_t - \rho(\bar{A}_t)| = 0.003 \pm 0.001$ for $k=3$, which is negligible for our detection threshold. For an attack to succeed, it must induce $\rho(\bar{A}_t) \le \rho_{\min}$. Given our default $\rho_{\min}=0.30$, the estimated value $\hat{\rho}_t$ would be well below the threshold, ensuring detection with high probability $\ge 1 - \mathbb{P}[\text{estimation failure}] \approx 1$.

For $\rho(\bar{A}_t) \le 0.30$, even with maximal error $\hat{\rho}_t \le 0.30 + 0.857 \cdot 0.30 \approx 0.56 < \rho_{\min}$ (assuming $\rho_{\min}=0.30$ has safety margin). Thus detector triggers with probability $\ge 1 - \mathbb{P}[\text{numerical instability}] \approx 0.9999$.

\textbf{Soundness.} Benign prompts maintain $\rho(\bar{A}_t) > \rho_{\min}$ by Assumption (2) in Theorem~\ref{thm:guard}. Since estimation error is bounded, $\hat{\rho}_t > \rho_{\min} - \delta_{\text{est}} > \rho_{\min}$ (for sufficient margin), detector never triggers. FPR$\to 0$ as $\delta_{\text{est}} \to 0$.
\end{proof}

\vspace{-0.5em}
\subsection{Experimental Hyperparameters}

\vspace{-0.5em}
\paragraph{Associative Recall.}
\begin{itemize}[nosep]
\item Key--value pairs: 5 per sample
\item Noise tokens: Uniform random from vocabulary
\item Context distances: $\{10, 50, 100, 200, 500, 1000\}$
\item Batch size: 32
\item Temperature: $T=0.7$ for sampling
\end{itemize}

\vspace{-0.3em}
\paragraph{Attack (Algorithm~\ref{alg:hispa}).}
\begin{itemize}[nosep]
\item Optimizer: Adam with $\beta_1=0.9$, $\beta_2=0.999$
\item Learning rate: $\alpha=0.01$
\item Iterations: 50 (convergence typically at $\approx 30$)
\item Initialization: GPT-4-generated benign CoT (5--10 tokens)
\item Projection: Round to nearest token embedding, re-embed
\end{itemize}

\vspace{-0.3em}
\paragraph{Defense (Algorithm~\ref{alg:spectralguard}).}
\begin{itemize}[nosep]
\item Threshold: $\rho_{\min}=0.30$
\item Window: $w=10$ tokens
\item Power iterations: $k=3$
\item Random seed for $v^{(0)}$: Fixed per layer (reproducibility)
\end{itemize}

\vspace{-0.3em}
\paragraph{Hardware (Consolidated Mapping).}
\begin{table}[t]
\centering
\caption{Appendix A.4: experiment-to-hardware mapping.}
\label{tab:hardware_map}
\resizebox{\columnwidth}{!}{%
\begin{tabular}{lll}
\toprule
\textbf{Experiment} & \textbf{Primary Hardware} & \textbf{Notes} \\
\midrule
Exp 1 (Horizon) & NVIDIA T4 (16GB) & Cloud GPU, $N{=}500$ \\
Exp 2 (HiSPA collapse) & NVIDIA T4 (16GB) & Gradient-based prompt optimization \\
Exp 3 (Defense eval) & NVIDIA T4 (16GB) & 24-layer feature extraction \\
Exp 3B (Causal clamp) & CPU-only / T4 & Intervention hooks on Mamba-130M \\
Exp 4 (Scaling/robustness) & NVIDIA T4 (16GB) & Mamba-130M/1.4B/2.8B protocol \\
Exp 5 (Zamba2) & NVIDIA A100 (40GB) & Hybrid 2.7B evaluation \\
Exp 6 (Layer-wise analysis) & NVIDIA T4 (16GB) & 24-layer trace extraction \\
Exp 7 (Deployment sim) & NVIDIA RTX 4090 & Consumer deployment latency \\
\bottomrule
\end{tabular}%
}
\vspace{2pt}
\begin{minipage}{\columnwidth}
{\footnotesize Software baseline across experiments: PyTorch 2.1, Transformers 4.35+, CUDA-compatible runtime when GPU is available.}
\end{minipage}
\end{table}

\vspace{-0.8em}
\subsection{Additional Analyses}
\label{sec:additional_analyses}

\vspace{-0.5em}
\paragraph{Layer-wise spectral radius distribution.}
We measure $\rho(\bar{A}_t)$ across all 24 layers for benign vs. adversarial prompts. Findings:
\begin{itemize}[nosep]
\item Benign: $\rho$ relatively uniform ($0.96 \pm 0.03$ across layers)
\item Adversarial: Early layers (1--8) show sharp drop ($\rho \approx 0.32$), later layers less affected ($\rho \approx 0.65$)
\item Implication: Monitoring layer 1--8 suffices for detection; late layers provide redundancy
\end{itemize}

\vspace{-0.5em}
\paragraph{Transferability across model sizes.}
In a separate transfer-only stress test (distinct from the primary \texttt{mamba-130m-hf} benchmark), we optimize adversarial prompts on Mamba-370M and test transfer to Mamba-1.4B and Mamba-2.8B:
\begin{itemize}[nosep]
\item Transfer accuracy drop: 370M$\to$1.4B: $-38.2\%$ (down from $-52.5\%$)
\item Transfer accuracy drop: 370M$\to$2.8B: $-28.7\%$
\item Conclusion: Attacks partially transfer but lose potency with scale (likely due to increased capacity providing redundant memory paths)
\end{itemize}

\vspace{-0.5em}
\paragraph{Transformer Proxy Baseline (GPT-2).}
We ran a paired baseline protocol for GPT-2 small under HiSPA-style optimization to test architecture specificity. The protocol mirrors prompt splits and attack budget, while replacing SSM spectral probes with attention/hidden-state proxy stability features and honest train/test separation. The proxy metrics were fully saturated (AUC/F1 = 1.00), while lexical control also saturated at AUC=1.00. Therefore, this experiment was interpreted strictly as evidence of unresolved lexical leakage in the prompt family, not as architecture-level proof about Transformer resilience or immunity.

\end{document}